\documentclass[12pt]{article}

\usepackage{amsmath}
\usepackage{times}
\usepackage{graphicx}
\usepackage{color}
\usepackage{amsfonts}
\usepackage{multirow}
\usepackage{epstopdf}
\usepackage{subfig}
\usepackage{booktabs}
\usepackage[square,numbers]{natbib}
\usepackage{amsthm}
\newtheorem{lemma}{Lemma}
\newtheorem{theorem}{Theorem}

\newtheorem{remark}{Remark}
\newtheorem{assumption}{Assumption}
\newtheorem{definition}{Definition}
\usepackage[linesnumbered,ruled]{algorithm2e}

\usepackage{rotating}
\usepackage{bbm}
\usepackage{latexsym}

\newcommand{\captionfonts}{\normalsize}

\makeatletter  
\long\def\@makecaption#1#2{%
  \vskip\abovecaptionskip
  \sbox\@tempboxa{{\captionfonts #1: #2}}%
  \ifdim \wd\@tempboxa >\hsize
    {\captionfonts #1: #2\par}
  \else
    \hbox to\hsize{\hfil\box\@tempboxa\hfil}%
  \fi
  \vskip\belowcaptionskip}
\makeatother   

    \def\independenT#1#2{\mathrel{\setbox0\hbox{$#1#2$}%
   \copy0\kern-\wd0\mkern4mu\box0}}

\begin{document}

\hspace{13.9cm}1

\ \vspace{20mm}\\
{\centering
{\LARGE  Confounder Detection in High Dimensional Linear Models using First Moments of Spectral Measures}
}
%

\thispagestyle{empty}
\markboth{}{NC instructions}
{\centering 
\ \vspace{-0mm}\\
{\bf \large Furui Liu, Laiwan Chan}\\
{Department of Computer Science and Engineering,}\\
{The Chinese University of Hong Kong}\\
}
\begin{center} {\bf Abstract} \end{center}
In this paper, we study the confounder detection problem in the linear model, where the target variable $Y$ is predicted using its $n$ potential causes $X_n=(x_1,...,x_n)^T$.    Based on an assumption of rotation invariant generating process of the model, recent study  shows that  the spectral measure induced by the regression coefficient vector with respect to the covariance matrix of $X_n$ is close to a uniform measure in purely causal cases, but it differs from a uniform measure characteristically in the presence of a scalar confounder. Then, analyzing spectral measure pattern could help to detect confounding. In this paper, we propose to use the first moment of the spectral measure for confounder detection. We calculate the first moment of the regression vector induced spectral measure, and compare it with the first moment of a uniform spectral measure, both defined with respect to the covariance matrix of $X_n$. The two moments coincide in non-confounding cases, and differ from each other  in the presence of confounding. This statistical causal-confounding asymmetry can be used for confounder detection. Without the need of analyzing the spectral measure pattern,  our method does avoid the difficulty of metric choice and multiple parameter optimization. Experiments on synthetic and real data show the performance of  this method.  

 \section{Introduction}
 In many real world applications, we  often face the problems of   estimating the causal effects of a set of sources $X_n=(x_1,...,x_n)^T$ on one target variable $Y$. To achieve this, one can build a linear regression model given their observations \citep{hoyer2008estimation, estimationSEM,lingam,directlingam,pwlingam}, and check the coefficients. For  variables that are statistically dependent, we would, in most of the cases, get non-zero regression coefficients between the observations\footnote{The regression coefficient here refers to the correlation coefficient between variables. It is known that dependent variables could also be uncorrelated, and in that case the regression coefficient is 0.}. If $x_j$ has a significant regression coefficient, it is believed to have a large causal influence on $Y$.  However, the correctness of this  is based on the causal sufficiency assumption that there is no hidden confounder of $X_n$ and $Y$, which cannot be verified from the regression procedure.  Simply checking   the coefficient vector cannot give us enough information for identifying confounder. Given the correctness of the causal sufficiency assumption unverified, estimating the causal effects by regression could be problematic: one never knows if the coefficients purely describe the influence of $X_n$ on $Y$, or it is significant because they share a hidden common driving force. Thus, confounder detection is important. It basically acts as a verification procedure of the causal sufficiency assumption. For further analysis, we write a mathematical model, and denote the non-observable confounder as $Z$. Directly following the paper \citep{janzing2017detecting}, we assume that the $Z$ is a one dimensional variable, and  consider the model
  \begin{eqnarray}
X_n & = & b_nZ+E_n,\\
Y & = & a_n^TX_n + cZ+ F,
 \end{eqnarray}
 where $a_n,b_n$ are $n$ dimensional vectors and $E_n$ is the $n$ dimensional noise. $F$ is the one dimensional noise, and $c$ is a scalar. When $\|b_n\|$ and $c$ are both non-zero\footnote{By default, $\|\cdot\|$ stands for the $L_2$ norm $\|\cdot\|_2$}, $Z$ is a confounder of  $X_n$ and $Y$ \citep{janzing2017detecting}.  Consider a least square regression of $Y$ on $X_n$ to get the regression coefficient as
    \begin{eqnarray}
  \tilde{a}_n & = & \Sigma_{X_n}^{-1} \Sigma_{X_nY},
  \end{eqnarray}
  The covariance matrices are 
  \begin{eqnarray*}
   \Sigma_{X_nY}  & = & (\Sigma_{E_n}+b_nb_n^T)a_n+cb_n, \\
  \Sigma_{X_n} & = &  \Sigma_{E_n}+b_nb_n^T.
     \end{eqnarray*}
 Notice that we assume the variance of variable $Z$ is 1 here. We consider this assumption justified since we can always make this true by rescaling $b_n$ and $c$. Then we get 
  \begin{eqnarray}
  \tilde{a}_n & = & a_n + c(\Sigma_{E_n}+b_nb_n^T)^{-1}b_n.
   \end{eqnarray}
    The regression coefficient basically consists of two parts. One is the part  describing the causal influences of $X_n$ on $Y$, and the other is the part describing  confounding effects. As this decomposition reveals, the regression coefficient in confounding and non-confounding cases could be clearly different. Consider the following points.
   \begin{enumerate}
   \item Purely causal cases:  $\|b_n\|$ or $c$ should be 0. In this case, $$\tilde{a}_n = a_n.$$
   \item Confounding cases:   $\|b_n\|$ and $c$ are not 0. In this case, $$\tilde{a}_n = a_n + c(\Sigma_{E_n}+b_nb_n^T)^{-1}b_n.$$
   \end{enumerate}
   For ease of explanation, we  denote $\tilde{a}_n$ as the composition of causal part and confounding part
   $$ \tilde{a}_n = \underbrace{a_n}_\text{causal part}+ \underbrace{c(\Sigma_{E_n}+b_nb_n^T)^{-1}b_n}_\text{confounding part}.$$ 
   When one conducts a regression, the obtained coefficients may contain  both parts.  Our goal is to tell if there is a confounder.
   
   A recent paper \citep{janzing2017detecting} proposes a method to achieve this goal. The core idea is based on the so-called generic orientation theory, motivated by recent advances in causal discovery that discuss certain independence between cause and mechanisms \citep{liu2016causal,janzing2010telling,liu2017causal,lemeire2013replacing}.  The method is built on a core term named the vector induced spectral measure with respect to $\Sigma_{X_n}$, which intuitively describes the squared length of the components of a vector projected into the  eigenspace of $\Sigma_{X_n}$. Later we would mention spectral measure multiple times and by default the spectral measure is induced with respect to $\Sigma_{X_n}$. Based on a rotation invariant model generating assumption and the concentration of measure phenomenon in high dimensional spheres \citep{marton1996bounding,talagrand1995concentration,shiffman2003random,popescu2006entanglement}, the paper \citep{janzing2017detecting} posts two asymptotic statements. First, the $a_n$ induced spectral measure and the $c(\Sigma_{E_n}+b_nb_n^T)^{-1}b_n$ induced spectral measure have their respective patterns. Second, the $\tilde{a}_n$ induced spectral measure is a direct summation of the two measures in first point.  Given the observed joint distribution of $Y$ and $X_n$, we can compute the $\tilde{a}_n$ induced spectral measure. Then, we use a convex combination of two spectral measures, one approximating $a_n$ induced spectral measure and the other approximating $c(\Sigma_{E_n}+b_nb_n^T)^{-1}b_n$ induced spectral measure, to match the observed measure. We tune the weights of the two measures and record the weights of the part approximating $c(\Sigma_{E_n}+b_nb_n^T)^{-1}b_n$ induced spectral measure  in the best match. The weight then, in certain sense, records the ``amount of confounding''. Although the confounding strength can be quantitatively estimated by this method, the drawback is still clear.
   \begin{enumerate}
   \item The two asymptotic statements are justified by weak convergence only, and the pattern approximations, as well as measure decomposition,  should be interpreted in a sufficiently loose sense. As a consequence, the total variation distance may fail to serve as a good metric, when one compares the reconstructed spectral measure with the observed one. However, the optimal choice of metrics stays vague and the method \citep{janzing2017detecting} depends on unjustified heuristic choices (kernel smoothing), which may lead to a wrong ``equal or not" conclusion. 
   \item The method needs to tune two parameters. One is the weight in reconstruction, and the other is a parameter related to approximations of the $c(\Sigma_{E_n}+b_nb_n^T)^{-1}b_n$ induced spectral measure. Optimizing over two-parameter space requires very fine-grained search, and error control is not easy. 
   \end{enumerate}
    These points reduce the reliability of the conclusions drew by the method.
   
   Can we identify confounding without reconstructing the whole spectral measure? This paper would provide an answer to this question. Recall that an important characteristic of a measure is its moment. We propose to directly use the moment information for confounder detection.  We would  focus on  the first moment, and show that the first moment of the  spectral measure induced by $\tilde{a}_n$ already behaves differently in causal and confounding scenarios. To access its ``behavior'' in a quantitatively concise sense, we later design a deviation measurement to quantify the difference between the first moment of the induced spectral measure and that of  a uniform measure.  The moment ``behaves differently'' is then justified by different asymptotic values of the  deviation measurement in causal and confounding cases. This statistic is easy to compute, and already provides us enough information for confounding detection.    Our method clearly avoids the aforementioned drawbacks of the method of the paper \citep{janzing2017detecting}.
   \begin{enumerate}
   \item Without the need of matching spectral measure patterns,  we do not need to tackle the vagueness of  ``interpreting the approximations loosely''  and the difficulty of metric choice. Instead, we compare the first moment of the spectral measure induced by $\tilde{a}_n$ with respect to $\Sigma_{X_n}$, with the first moment of the uniform (tracial) spectral measure on $\Sigma_{X_n}$, and make conclusions based on their differences. 
   \item The parameter we need is a threshold of the deviation measurement. Simultaneous optimizing two parameters, as the spectral measure pattern matching method \citep{janzing2017detecting} does, is avoided.
   \end{enumerate} 
   
   These justify the usability of our proposed  method, and  it might provide a better solution than the existing one. We will detail our method  in the following sections, and present theoretical and empirical analysis. To begin with, we first describe the related work.
\section{Related work}
 We describe the  method by \citep{janzing2017detecting} for confounder detection. The basic idea is that the causal part and confounding part have their own features in induced spectral measures with respect to the covariance matrix of the cause, such that  each part can be approximated and combined.  To understand this, we first give some basic definitions.
 \begin{definition}[Eigendecomposition]
The eigendecomposition of $\Sigma_{X_n}$ is 
\begin{equation}
\Sigma_{X_n}=U_n\Lambda_n U_n^T = \sum_{i=1}^n \lambda_i u_iu_i^T,
\end{equation}
where the matrices 
\begin{eqnarray}
U_n  & = & [u_1,...,u_n], \\
\Lambda_n & = & diag(\lambda_1,...,\lambda_n),
\end{eqnarray}
are  matrices of eigenvectors and the diagonal matrix containing all eigenvalues respectively.
 \end{definition}

  \begin{definition}[Vector induced spectral measure]
   A spectral measure induced by an $n$ dimensional vector $\phi_n$ with respect to covariance matrix  $\Sigma_{X_n}$ in (5) is defined as
   \begin{eqnarray}
   \mu_{\Sigma_{X_n},\phi_n} & = & \sum_{i=1}^{n} <\phi_n,u_i>^2\delta_{\lambda_i},
   \end{eqnarray} 
   where $\lambda_i$ belonging to the $\Lambda_n$ in (7) is the $i$th eigenvalue and $u_i$ is respective eigenvector. $\delta_s$ is the point measure defined on the $s \in \mathbb{R}$.
    \end{definition}
  Later on we would also need a tracial (uniform) spectral measure, and we here define it formally. 
  \begin{definition}[Tracial spectral measure]
     A normalized tracial spectral measure defined on the covariance matrix  $\Sigma_{X_n}$ in (5) is as
      \begin{eqnarray}
         \mu_{\Sigma_{X_n}}^\tau & = & \frac{1}{n}\sum_{i=1}^{n}  \delta_{\lambda_i},
         \end{eqnarray} 
       where $\delta_s$ is the point measure on $s \in \mathbb{R}$.
      \end{definition}
       \begin{remark}
         We also call the tracial spectral measure  ``uniform spectral measure''. It should be noted that  we no longer make the non-degenerate assumption on $\Sigma_{X_n}$, which is made by \citep{janzing2017detecting}. Thus, the statement of ``uniform'' should be interpreted in a more general sense instead of uniformly spreading over a domain in $\mathbb{R}$: the weight of each point measure is equal while the point measures are allowed to overlap with each other. 
        \end{remark}
      
 For  numerical computations (matching spectral measure pattern),  the induced spectral measure can be represented using two vectors. The first one is the vector containing its support $\lambda_{\Sigma_{X_n}} = (\lambda_1,...,\lambda_n)^T$, and the second one is a vector $ \omega_{\Sigma_{X_n},\phi_n}$ containing the values of the spectral measure at those support points. We formally give the definitions below.
   \begin{definition}[Vectorized representation of spectral measure]
     For the vector induced spectral measure $ \mu_{\Sigma_{X_n},\phi_n}$ in  (8), we use two vectors to represent it.
       \begin{eqnarray}
          \lambda_{\Sigma_{X_n}} & = & (\lambda_1,...,\lambda_n)^T,\\
          \omega_{\Sigma_{X_n},\phi_n} & = & (<\phi_n,u_1>^2,...,<\phi_n,u_n>^2)^T.
          \end{eqnarray} 
          For tracial spectral measure, we can use  
           \begin{eqnarray}
                    \lambda_{\Sigma_{X_n}} & = & (\lambda_1,...,\lambda_n)^T,\\
                    \omega_{\Sigma_{X_n}}^\tau & = & (\frac{1}{n},...,\frac{1}{n})^T.
                    \end{eqnarray} 
       \end{definition}
 
 The $i$th element $\omega_{\Sigma_{X_n},\phi_n} (i)$  in the vector records the value of the spectral measure at $\lambda_i$. One can compute the vector as
$$ \omega_{\Sigma_{X_n},\phi_n}  = U_n^T\phi_n \circ U_n^T\phi_n,$$  
where $\circ$ denotes the Hadamard product.  An intuitive understanding, is that it is a vector describing the squared coordination of the vector $\phi_n$ with respect to the eigenspace of  $\Sigma_{X_n}$.  
This vector records the pattern of the spectral measure, and is used for computational tasks like ``pattern matching''.  Now consider $\mu_{\Sigma_{X_n},\tilde{a}_n } $, which is the spectral measure induced by the regression vector $\tilde{a}_n $ with respect to the covariance matrix  $\Sigma_{X_n}$.   The paper \citep{janzing2017detecting} shows that given $n$ is large, this spectral measure can be decomposed as $$\mu_{\Sigma_{X_n},\tilde{a}_n } \approx \underbrace{\mu_{\Sigma_{X_n},a_n } }_\text{causal part} +  \underbrace{\mu_{\Sigma_{X_n},c(\Sigma_{E_n}+b_nb_n^T)^{-1}b_n}}_\text{confounding part},$$ with so-called general orientation theory. As a consequence, their vectorized representations have the property
$$\omega_{\Sigma_{X_n},\tilde{a}_n } \approx \underbrace{\omega_{\Sigma_{X_n},a_n } }_\text{causal part} +  \underbrace{\omega_{\Sigma_{X_n},c(\Sigma_{E_n}+b_nb_n^T)^{-1}b_n}}_\text{confounding part}.$$
  Then given the observations, the confounding can be estimated using the following steps.
  \begin{enumerate}
  \item Approximate the  spectral measure induced by causal part as $\omega_{\Sigma_{X_n},a_n } = \omega_{\Sigma_{X_n}}^\tau$ in (13).
 \item   Approximate the spectral measure induced by confounding part as
   $$\omega_{\Sigma_{X_n},c(\Sigma_{E_n}+b_nb_n^T)^{-1}b_n}^\nu  =   \frac{1}{\|H_\nu^{-1}\mathbf{1}_n\|^2}\omega_{H_\nu,H^{-1}_\nu\mathbf{1}_n}, ~~~~H_\nu  = \Lambda_n + \frac{\nu}{n} \mathbf{1}_n\mathbf{1}_n^T,$$
    and $\omega_{H_\nu,H^{-1}_\nu\mathbf{1}_n}$ is the vectorized representation of the spectral measure induced by $H^{-1}_\nu\mathbf{1}_n$ with respect to $H_\nu$. $\Lambda_n$ is the matrix in (7). $\mathbf{1}_n$ is the vector of all 1s, and $\nu$ is a parameter.
 \item    Compute $\frac{1}{\|\hat{\tilde{a}}_n\|^2}\hat{\omega}_{\Sigma_{X_n},\tilde{a}_n }$ using the observations, and find the parameters $\beta^*,\nu^*$ that minimize a reconstruction error as
   $$(\beta^*,\nu^*) =  \arg\min_{\beta,\nu} \|\frac{1}{\|\hat{\tilde{a}}_n\|^2}\hat{\omega}_{\Sigma_{X_n},\tilde{a}_n } - (1- \beta) \omega_{\Sigma_{X_n},a_n}  -  \beta\omega_{\Sigma_{X_n},c(\Sigma_{E_n}+b_nb_n^T)^{-1}b_n}^\nu \|_K$$
   where $\|\cdot\|_K$ is a  kernel smoothed metric \citep{janzing2017detecting}.
   \end{enumerate}  
The ``confounding or not'' conclusion then relies on  $\beta^*$. If $\beta^*$ is significant, then it shows a clear confounding effect. As we mentioned, the method has clear drawback. The weak convergence property in infinite dimensions makes the practical success of this method heavily depend on a good distance metric.  To make it easy to understand,   consider an example of purely causal models. We generate the coefficients vector $a_n$ of dimension 10, with each entry uniformly drawn from $[-0.5,0.5]$. Then we normalize it to unit norm and calculate the induced spectral measure (vectorized representation) with respect to a random covariance matrix.\footnote{Here the random covariance matrix is generated by $\Sigma = 0.5*(A_n+A_n^T)$, where $A_n$ is an $n$ dimensional matrix and  the elements in $A_n$ are drawn  from a uniform distribution on $[-0.5,0.5]$. Then we extract its orthogonal bases $V_n$ and generate a diagonal matrix $\Gamma_n$ with entries sampled from a uniform $(0,1)$. Lastly output $V_n\Gamma_nV_n^T$.} 
  \begin{figure}[h]%
      \centering
    \subfloat[Practical spectral measure] {{\includegraphics[width=.5\textwidth,height =.4\textwidth ]{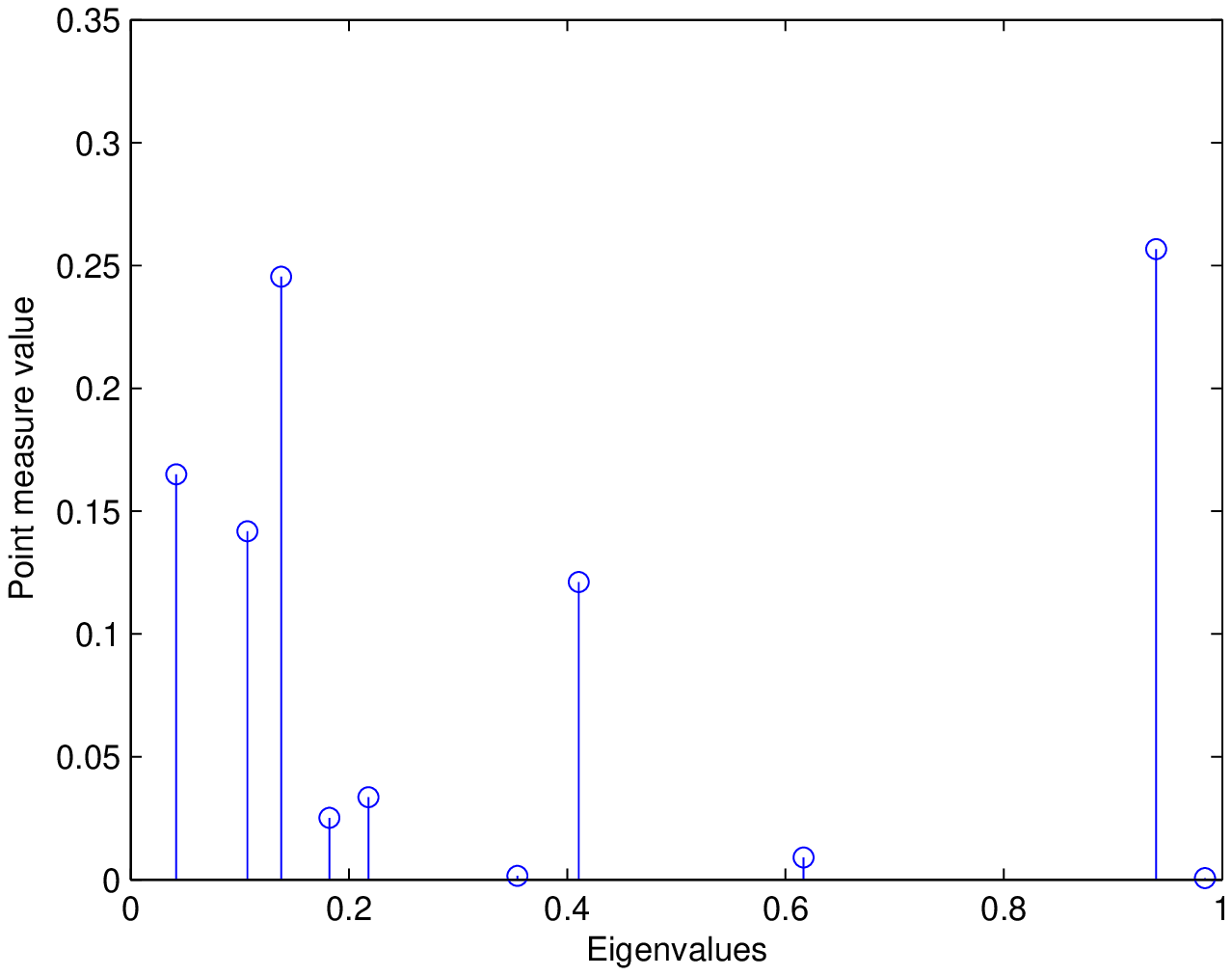} }}
    \subfloat[Tracial spectral measure]{{\includegraphics[width=.5\textwidth,height =.4\textwidth]{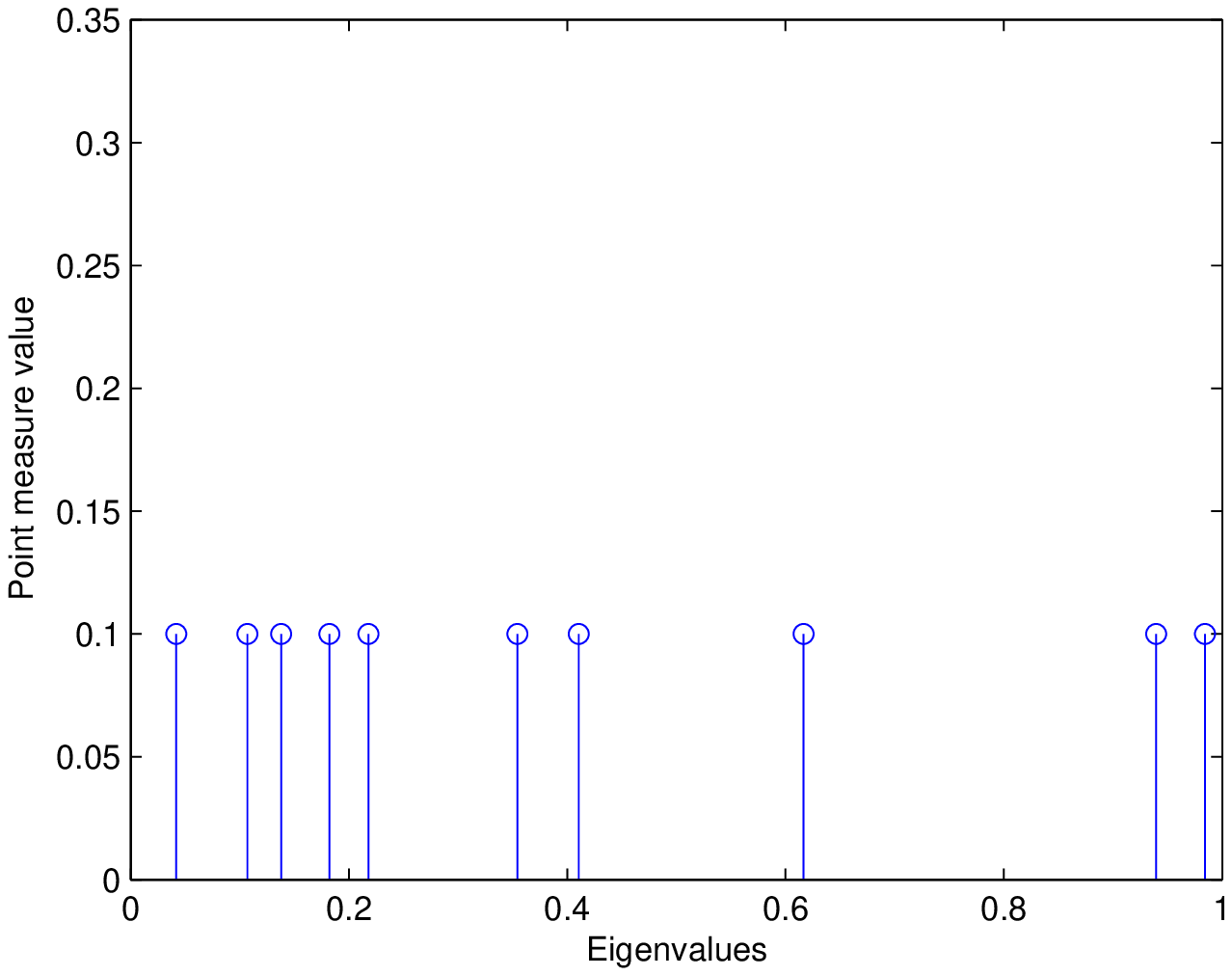} }}  
   \caption{Practical spectral measure induced by causal part and the uniform one}
    \end{figure}
One can  see that the practical spectral measure has a large difference from the uniform one. When one wants to match the two patterns (vectorized representation) and concludes ``purely causal'', one should adjust weights  on different dimensions. However, the optimal choice of the metric remains unknown in the pattern matching method \citep{janzing2017detecting}. It only relies on eigenvalue-gap related heuristic kernel smoothing. In this example, the kernel smoothing matrix would not be a good choice, since the eigenvalue-gaps are quite random compared to the spectral pattern. One may get wrong conclusions, that the pattern in figure 1(a) differs a lot from a uniform measure, and it is a confounding model. 

In summary, the method \citep{janzing2017detecting} relies on analyzing the patterns of the spectral measure. In purely causal cases, it is a uniform one. But the presence of confounding vector ``modifies'' the pattern in a characteristic way and can be detected. However, the weak convergence property makes it very hard to choose a metric for comparing the reconstructed measure with the practical one, thus hindering a good understanding of the pattern. Since reconstructing confounding by measure approximation and combination is really a hard task,  why not directly check the moment information? In this way, we avoid the hardness of metric choice and multiple parameter optimization, that one has to face when  pattern matching method is used. We would later focus on the first moment, and show that checking first moment is enough for us to identify a confounder. This is because, asymptoticly, the first moment of the measure induced by the regression vector coincides with that of a uniform measure in purely causal cases, while it does not in confounding cases. This causal-confounding asymmetry could help us identify the confounding. We put down thorough discussions about the identifiability of the confounder using the first moment. To begin with, we start with the definitions related to first moment of spectral measures.

\section{First moments of spectral measures}   
We define the first moment of the measures here, and design a moment deviation measurement to test the moment behavior of a vector induced spectral measure. We first start with some definitions.
  \begin{definition}[First moment]
  The first moment of the vector induced spectral measure $\mu_{\Sigma_{X_n},\phi_n}$ in (8) is defined as 
  \begin{eqnarray}
           M(\mu_{\Sigma_{X_n},\phi_n}) & = & \int_{\mathbb{R}}s~d\mu_{\Sigma_{X_n},\phi_n}(s),
           \end{eqnarray}
  and the first moment of tracial spectral measure $\mu_{\Sigma_{X_n}}^\tau$  in (9) is defined as
      \begin{eqnarray}
            M( \mu_{\Sigma_{X_n}}^\tau) & = & \int_{\mathbb{R}}s~d \mu_{\Sigma_{X_n}}^\tau(s).
            \end{eqnarray} 
  \end{definition}
   For practical computations, one can use vectorized representation of the measures to compute the moments.
    \begin{definition}[First moment using vectorized representation]
    The first moments of the vector induced spectral measure $\mu_{\Sigma_{X_n},\phi_n}$ in (8) and the tracial spectral measure $\mu_{\Sigma_{X_n}}^\tau$  in (9) can be written using vectorized representation  as
        \begin{eqnarray}
           M(\mu_{\Sigma_{X_n},\phi_n}) & = & \lambda_{\Sigma_{X_n}}^T\omega_{\Sigma_{X_n},\phi_n} = \sum_{i=1}^n \lambda_i <\phi_n,u_i>^2,\\
              M( \mu_{\Sigma_{X_n}}^\tau) & = & \lambda_{\Sigma_{X_n}}^T\omega_{\Sigma_{X_n}}^\tau = \frac{1}{n}\sum_{i=1}^n \lambda_i,
              \end{eqnarray}
               respectively.
               \end{definition}
              
   From the equation (16), we can rewrite it as
   $$\sum_{i=1}^n \lambda_i <\phi_n,u_i>^2 = \sum_{i=1}^n \lambda_i \phi_n^Tu_iu_i^T\phi_n = \phi_n^T(\sum_{i=1}^n \lambda_i u_iu_i^T)\phi_n.$$
   In other words, we can write the moments as a multiplication of vector and matrix. We have the following definition. 
   \begin{definition}[First moment using $\phi_n$ and $\Sigma_{X_n}$]
   The first moment of the vector induced spectral measure in (16) can be written as
              \begin{eqnarray}
                          M(\mu_{\Sigma_{X_n},\phi_n}) =    \phi_n^T\Sigma_{X_n}\phi_n. 
                 \end{eqnarray} 
              If we then define  a renormalized trace as 
              \begin{eqnarray}
              \tau_n(\cdot) = \frac{1}{n}tr(\cdot),
              \end{eqnarray}
              then (17) can be written as
              \begin{eqnarray}
               M( \mu_{\Sigma_{X_n}}^\tau)  =    \tau_n(\Sigma_{X_n}) .
                          \end{eqnarray}  
    \end{definition} 
   As we previously sketched, we want to design a measure to quantify the difference between the first moment of vector induced spectral measure and the first moment of a tracial spectral measure. Since the tracial spectral measure is a normalized one while the vector induced spectral measure enlarges with the norm of the vector, we also enlarge the tracial measure with the norm. We have the following definition.
   \begin{definition}[First moment deviation]
      The deviation of the first momentof the $\phi_n$ induced spectral measure on $\Sigma_{X_n}$ from that of a tracial spectral measure on $\Sigma_{X_n}$ is defined as 
                 \begin{eqnarray}
                         \mathcal{D}(\phi_n,\Sigma_{X_n}) = |M(\mu_{\Sigma_{X_n},\phi_n}) -  \|\phi_n\|^2M( \mu_{\Sigma_{X_n}}^\tau) |.  
                    \end{eqnarray} 
                 Using definition 7, one can also write it as  
                 \begin{eqnarray}
              \mathcal{D}(\phi_n,\Sigma_{X_n}) = |\phi_n^T\Sigma_{X_n}\phi_n  - \|\phi_n\|^2\tau_n(\Sigma_{X_n})|.
                 \end{eqnarray}
                 Since we only measure first moment deviations of induced measures defined on $\Sigma_{X_n}$, without causing confusions we write 
                 \begin{eqnarray}
               \mathcal{D}(\phi_n) = \mathcal{D}(\phi_n,\Sigma_{X_n}). 
                             \end{eqnarray}  
       \end{definition}
       Later we mainly study the asymptotic behavior of the deviation measurement. For ease of representation, we define the asymptotic deviation measurement below.
       \begin{definition}[Asymptotic $\mathcal{D}(\phi_n)$]
       For a sequence of vectors $\{\phi_n\}_{n \in \mathbb{N}}$ and a sequence of positive semi-definite $n\times n$ symmetric matrices $\{\Sigma_{X_n}\}_{n \in \mathbb{N}}$, the asymptotic $\mathcal{D}(\phi_n)$ is defined as
       $$ \mathcal{D}(\phi_\infty) = \lim_{n \to \infty}\mathcal{D}(\phi_n),$$
       given the limit exists.
       \end{definition}
       Now we get everything ready to proceed. Recall the linear confounding model defined in (1) and (2). Given observational data, we can compute the covariance matrix and the regression vector $\tilde{a}_n$, and thus the induced spectral measure on $\Sigma_{X_n}$ and the tracial one.  As we mentioned, we want to show the performance of the deviation measure in (23) in causal and confounding cases. We begin with  causal cases.  
\section{$\mathcal{D}(\tilde{a}_\infty)$ in causal cases}
In this section, we describe our method starting from  the properties of the causal cases. In this case, $c\|b_n\| = 0$. Then $\tilde{a}_n = a_n$ and it is independently chosen with respect to $\Sigma_{X_n}$. This ``independence between cause and mechanism'' concept is then realized statistically by an assumption of generative model. In functional models, one often considers the property of the noise, like the independence between noise and cause \citep{lingam}, or certain invariance of the number of support points of the conditional distributions \citep{liu2016causald}.   To begin with, we first put down this generative model assumption.
\begin{assumption}
Let $\{a_n\}_{n \in \mathbb{N}}$ be a sequence of vectors drawn uniformly at random from a sphere in $\mathbb{R}^n$ with fixed radius $r_a$. $\{\Sigma_{X_n}\}_{n \in \mathbb{N}}$ is a uniformly bounded sequence of positive semi-definite $n\times n$ symmetric matrices, and the tracial spectral measure converges weakly as
\begin{equation}
 \lim_{n \to \infty} \mu_{\Sigma_{X_n}}^\tau = \mu^\infty.
\end{equation}
\end{assumption}
Now we quote a lemma from the paper \citep{janzing2017detecting}.
\begin{lemma}[Weak measure convergence]
Let assumption 1 be satisfied, and the tracial measure converges as described by (24). Then for the model described in (1) and (2), we have  
\begin{equation}
 \lim_{n \to \infty} \mu_{\Sigma_{X_n},a_n} = r_a^2\mu^\infty
\end{equation}
weakly in probability.
\end{lemma}
Now we can  have a theorem of the asymptotic behavior of the deviation measure in causal cases. 
\begin{theorem}[Asymptotic first moment coincidence]
Let assumption 1 be satisfied. For the model described in (1) and (2),  we have  
\begin{equation}
 \mathcal{D}(a_\infty) = 0
\end{equation}
weakly in probability.
\end{theorem}
\begin{proof}
Using the lemma 1, we have
\begin{eqnarray*}
 \mathcal{D}(a_\infty)  & = & \lim_{n \to \infty} |M(\mu_{\Sigma_{X_n},a_n}) -  \|a_n\|^2M( \mu_{\Sigma_{X_n}}^\tau) |,   \\
& = &  |M(r_a^2\mu^\infty) -  r_a^2M( \mu^\infty) |, \\
& = & |r_a^2M(\mu^\infty) -  r_a^2M( \mu^\infty) |,\\
& = & 0.
\end{eqnarray*}
\end{proof}
Interestingly, if we write the deviation measurement using (22), as
$$ \mathcal{D}(a_n) = |a_n^T\Sigma_{X_n}a_n  - \|a_n\|^2\tau_n(\Sigma_{X_n})|,$$
we can draw a link from this to trace condition in the following remark. 
  \begin{remark}
  The condition in theorem 1 is equivalent to the trace condition \citep{janzing2010causal} when restricting $Y$ to be  1 dimensional. The trace condition is
  $$ \lim_{n \to \infty} \tau_1(a_n^T\Sigma_{X_n}a_n) = \lim_{n \to \infty} \tau_n(\Sigma_{X_n})\tau_1(a_na_n^T) = \lim_{n \to \infty} \|a_n\|^2\tau_n(\Sigma_{X_n}).$$
 Assumption 1 in trace method is stated as: $\Sigma_{X_n}$ is drawn from a distribution that is invariant under orthogonal transformation $\Sigma_{X_n} \mapsto V_n\Sigma_{X_n}V_n^T$ where $V_n$ is any $n$ dimensional unitary matrix. Consequently, $V_n^Ta_n$ is a random point uniformly drawn on some sphere.
  \end{remark}

  We has analyzed the behavior of the deviation measurement in causal cases. It is of interest that how this deviation measurement would behave in the presence of a confounder. We present the analysis about this in the next section.
  \section{$ \mathcal{D}(\tilde{a}_\infty) $ in  confounding cases}
  Recall the equation of the regression coefficient in the first section as
    $$ \tilde{a}_n = a_n+ c(\Sigma_{E_n}+b_nb_n^T)^{-1}b_n.$$ 
  In confounding cases, the second part is no longer 0. We want to know the value of the deviation measure $\mathcal{D}(\tilde{a}_n)$ in this case. If $\mathcal{D}(\tilde{a}_n)$ is no longer 0, then we already get the feature for confounding detection. In fact, this is true. When no confounder exists, the regression coefficient $\tilde{a}_n=a_n$ and it is assumed to be drawn uniformly at random on some high dimensional sphere. As a consequence, the  spectral measure induced by the vector has exactly the same asymptotic behavior as the tracial measure in terms of moments, and thus justifies a ``generic orientation'' of the causal part in the eigenspace of $\Sigma_{X_n}$. However, when we have confounder, this ``independence'' no longer holds. The expression
     $$ \tilde{a}_n = \underbrace{a_n}_\text{causal part}+ \underbrace{c(\Sigma_{E_n}+b_nb_n^T)^{-1}b_n}_\text{confounding part}.$$ 
   tells clearly, that the second part typically spoils the generic orientation of the vector $\tilde{a}_n$ with respect to the eigenspace of $\Sigma_{X_n}$ \citep{janzing2017detecting}.  We would then analyze more about its property. We start with model assumptions.
  \subsection{Rotation invariant generating model}
  As discussed before, the independence assumption between $\tilde{a}_n$ and $\Sigma_{X_n}$ no longer holds. In analogy to assumption 1, we could still make some assumptions on the confounding model. Note that later we would need to define vector induced and tracial spectral measure on $\Sigma_{E_n}$, which is in the same way as defining them on $\Sigma_{X_n}$. Consider the 2 points listed below. 
  \begin{assumption}
   $\{\Sigma_{E_n}\}_{n \in \mathbb{N}}$ and its inverse sequences $\{\Sigma_{E_n}^{-1}\}_{n \in \mathbb{N}}$ and $\{\Sigma_{E_n}^{-2}\}_{n \in \mathbb{N}}$  are uniformly bounded sequences of positive semi-definite $n\times n$ symmetric matrices, and their tracial spectral measures converge weakly as 
    \begin{eqnarray}
    \lim_{n \to \infty} \mu_{\Sigma_{E_n}}^\tau  & =&  \mu^\infty_1,\\
  \lim_{n \to \infty} \mu_{\Sigma_{E_n}^{-1}}^\tau &  = &  \mu_{-1}^\infty ,\\
  \lim_{n \to \infty} \mu_{\Sigma_{E_n}^{-2}}^\tau &  = &  \mu_{-2}^\infty,
    \end{eqnarray}
      respectively.
   \end{assumption}
   
  \begin{assumption}
 $\{a_n\}_{n \in \mathbb{N}}$ is a sequence of vectors drawn uniformly at random from a sphere in $\mathbb{R}^n$ with fixed radius $r_a$.  $\{b_n\}_{n \in \mathbb{N}}$ is a sequence of vectors drawn uniformly at random from a sphere in $\mathbb{R}^n$ with fixed radius $r_b$.
  \end{assumption}

 We would use these to help us refine $\mathcal{D}(\tilde{a}_\infty)$. Before we proceed, we list some core lemmas that are useful in the derivation of the asymptotic form of $\mathcal{D}(\tilde{a}_n)$.
  
  \subsection{Core lemmas}
  We here list some core lemmas that are useful for our analysis. They are true when the above model assumptions are satisfied. Some might be directly from \citep{janzing2017detecting}. 
  
  \begin{lemma}[Asymptotic norm decomposition]
    Let assumption 2 and 3 be satisfied, and we have
    $$ \lim_{n \to \infty} \|\tilde{a}_n\|^2 = \lim_{n \to \infty}\|a_n\|^2 + \lim_{n \to \infty}\|c\Sigma_{X_n}^{-1}b_n\|^2 $$
    almost surely.
    \end{lemma}
  \begin{proof}
  \begin{eqnarray*}
 \lim_{n \to \infty} \|\tilde{a}_n\|^2 & = &  \lim_{n \to \infty} \|a_n + c\Sigma_{X_n}^{-1}b_n\|^2,\\
              & = & \lim_{n \to \infty}\|a_n\|^2 + \lim_{n \to \infty}\|c\Sigma_{X_n}^{-1}b_n\|^2+ \lim_{n \to \infty} 2ca_n^T\Sigma_{X_n}^{-1}b_n, \\
      & = &\lim_{n \to \infty}\|a_n\|^2 + \lim_{n \to \infty}\|c\Sigma_{X_n}^{-1}b_n\|^2.
      \end{eqnarray*}
        The last item vanishes because $a_n$ is a point sequence uniformly chosen at random, and it is independently chosen with $\Sigma_{X_n}^{-1}b_n$. Thus, they are asymptotically orthogonal almost surely.  
  \end{proof}
    \begin{lemma}[Asymptotic moment decomposition]
      Let assumption 2 and 3 be satisfied, and we have
      \begin{equation}
       \lim_{n \to \infty} M(\mu_{\Sigma_{X_n},\tilde{a}_n }) = \lim_{n \to \infty}M(\mu_{\Sigma_{X_n},a_n }) + \lim_{n \to \infty}M(\mu_{\Sigma_{X_n},c\Sigma_{X_n}^{-1}b_n}),
      \end{equation}
      weakly in probability.
      \end{lemma}
      This lemma 3 is a direct result of the measure decomposition by paper \citep{janzing2017detecting} and we do not need to prove. 
          \begin{lemma}[Asymptotic  moment of tracial spectral measure of rank one perturbation]
            Let assumption 2 and 3 be satisfied. Recall the equation
            $$\Sigma_{X_n} =  \Sigma_{E_n} + b_nb_n^T,$$
             and we have
            \begin{equation}
             \lim_{n \to \infty} M(\mu_{\Sigma_{X_n}}^\tau) = M(\mu_{1 }^\infty) + \tau_{\infty}(r_b^2),
            \end{equation}
            weakly in probability.
            \end{lemma}
     \begin{proof}
       \begin{eqnarray*}
      \lim_{n \to \infty} M(\mu_{\Sigma_{X_n}}^\tau) & = &  \lim_{n \to \infty} \tau_n(\Sigma_{E_n} + b_nb_n^T),\\
                   & = & \lim_{n \to \infty}\tau_n(\Sigma_{E_n}) + \lim_{n \to \infty}\tau_n(b_nb_n^T), \\
           & = & M(\mu_{1 }^\infty) + \tau_{\infty}(r_b^2).
           \end{eqnarray*}
       \end{proof} 
       Note that $\tau_{\infty}(r_b^2)$ alone should be 0 here. However, we keep this term because later it might multiply with an unbounded term in some case study, and it cannot be simply ignored.
      \begin{lemma}[Moment convergence]
      Let assumption 2 and 3 be satisfied. We have
     \begin{eqnarray}
    \lim_{n \to \infty} M(\mu_{\Sigma_{X_n},a_n})    & = & \lim_{n \to \infty} \|a_n\|^2M(\mu_{\Sigma_{X_n}}^\tau) = r_a^2 M(\mu_{1 }^\infty) +  r_a^2\tau_{\infty}(r_b^2),\\ 
\lim_{n \to \infty} M(\mu_{\Sigma_{E_n}^{-1},b_n})    & = & \lim_{n \to \infty} \|b_n\|^2M(\mu_{\Sigma_{E_n}^{-1}}^\tau) = r_b^2 M(\mu_{-1 }^\infty), \\ 
\lim_{n \to \infty} M(\mu_{\Sigma_{E_n}^{-2},b_n})    & = & \lim_{n \to \infty} \|b_n\|^2M(\mu_{\Sigma_{E_n}^{-2}}^\tau) = r_b^2 M(\mu_{-2 }^\infty),
      \end{eqnarray} 
      weakly in probability.
      \end{lemma} 
      This lemma 5 is a direct result of lemma 1 and lemma 4.
        \begin{remark}
        These lemmas actually are based on a rotation invariant generating process of the model, which may be violated in some practical scenarios. Bear in mind that weaker assumptions may still lead to the same identities in these statements. This is because the geometry of high dimensional sphere makes majority of the vectors close to their center, which admits ``moment concentration''. This is also mentioned in paper \citep{janzing2017detecting}.
        \end{remark} 
        Finally we need another formula to help us in the proof of the theorem later. 
  \begin{lemma}[Sherman-Morrison formula]
  Sherman-Morrison formula for expressing inverse of rank one perturbation of a matrix \citep{bunch1978rank} is
  $$(\Sigma_{E_n}+b_nb_n^T)^{-1} = \Sigma_{E_n}^{-1} - \frac{\Sigma_{E_n}^{-1}b_nb_n^T\Sigma_{E_n}^{-1}}{1+b_n^T\Sigma_{E_n}^{-1}b_n}.$$
  \end{lemma}

  Now we get everything ready. We would show then how these lemmas would help us to refine the asymptotic expression of the first moment deviation to a concise form,  which could be used for other analysis. 
  \subsection{ $\mathcal{D}(\tilde{a}_\infty)$}
  In this section, we would give the asymptotic $\mathcal{D}(\tilde{a}_n)$ in confounding cases, with a detailed proof. We formalize it using the theorem below.
  \begin{theorem}[Asymptotic first moment deviation]
  Consider the confounding model described by (1) and (2). Let assumption 2 and 3 be satisfied. We have
  \begin{equation}
   \mathcal{D}(\tilde{a}_\infty) = \frac{c^2r_b^2}{(1+r_b^2 M(\mu_{-1 }^\infty))^2}|M(\mu_{-1 }^\infty) - M(\mu_{-2 }^\infty)M(\mu_{1 }^\infty)+r_b^2M(\mu_{-1 }^\infty)^2-\tau_{\infty}(r_b^2)M(\mu_{-2}^\infty)|
  \end{equation}
  weakly in probability.
  \end{theorem}
  \begin{proof}
 We would use those lemmas to proceed.
 \begin{eqnarray*}
 \lim_{n \to \infty} \mathcal{D}(\tilde{a}_n) & = & \lim_{n \to \infty} |M(\mu_{\Sigma_{X_n},\tilde{a}_n}) -  \|\tilde{a}_n\|^2M( \mu_{\Sigma_{X_n}}^\tau)|,\\
 & = & |\lim_{n \to \infty}M(\mu_{\Sigma_{X_n},\tilde{a}_n}) -  \lim_{n \to \infty}\|\tilde{a}_n\|^2M( \mu_{\Sigma_{X_n}}^\tau)|\\ 
 & \stackrel{lemma (3)}{=} &   |\lim_{n \to \infty}M(\mu_{\Sigma_{X_n},a_n}) + \lim_{n \to \infty}M(\mu_{\Sigma_{X_n},c\Sigma_{X_n}^{-1}b_n}) -  \lim_{n \to \infty}\|\tilde{a}_n\|^2M( \mu_{\Sigma_{X_n}}^\tau)|,
 \end{eqnarray*}
 where the ``lemma(3)'' refers to the fact that lemma 3 is used here.  Then we proceed to refine $\mathcal{D}(\tilde{a}_n )$.
  \begin{eqnarray*}
  \lim_{n \to \infty} \mathcal{D}(\tilde{a}_n) & = &   |\lim_{n \to \infty}M(\mu_{\Sigma_{X_n},a_n}) + \lim_{n \to \infty}M(\mu_{\Sigma_{X_n},c\Sigma_{X_n}^{-1}b_n}) -  \lim_{n \to \infty}\|\tilde{a}_n\|^2M( \mu_{\Sigma_{X_n}}^\tau)|, \\
  & \stackrel{lemma (5)}{=} &  |\lim_{n \to \infty}M(\mu_{\Sigma_{X_n},c\Sigma_{X_n}^{-1}b_n}) + \lim_{n \to \infty}\|a_n\|^2M(\mu_{\Sigma_{X_n}}^\tau) -  \lim_{n \to \infty}\|\tilde{a}_n\|^2M( \mu_{\Sigma_{X_n}}^\tau)| ,\\
  & = & |\lim_{n \to \infty}M(\mu_{\Sigma_{X_n},c\Sigma_{X_n}^{-1}b_n}) - \lim_{n \to \infty}(\|\tilde{a}_n\|^2 - \|a_n\|^2)M( \mu_{\Sigma_{X_n}}^\tau)|,\\
   & = & |\lim_{n \to \infty}M(\mu_{\Sigma_{X_n},c\Sigma_{X_n}^{-1}b_n}) - \lim_{n \to \infty}(\|\tilde{a}_n\|^2 - \|a_n\|^2)\lim_{n \to \infty}M( \mu_{\Sigma_{X_n}}^\tau)|,\\
  & \stackrel{lemma (2)}{=} &  |\lim_{n \to \infty}M(\mu_{\Sigma_{X_n},c\Sigma_{X_n}^{-1}b_n}) - \lim_{n \to \infty} (\|a_n\|^2 + \|c\Sigma_{X_n}^{-1}b_n\|^2-\|a_n\|^2 )\lim_{n \to \infty}M( \mu_{\Sigma_{X_n}}^\tau)|, \\
  & = & | \lim_{n \to \infty}M(\mu_{\Sigma_{X_n},c\Sigma_{X_n}^{-1}b_n}) -  \lim_{n \to \infty}\|c\Sigma_{X_n}^{-1}b_n\|^2 \lim_{n \to \infty}M( \mu_{\Sigma_{X_n}}^\tau)|.
  \end{eqnarray*}
 Recall the equation 
 $$ \Sigma_{X_n} = \Sigma_{E_n} + b_nb_n^T.$$
 Using definition 7, we have 
 \begin{eqnarray*}
M(\mu_{\Sigma_{X_n},c\Sigma_{X_n}^{-1}b_n}) & = & (c\Sigma_{X_n}^{-1}b_n)^T \Sigma_{X_n} c\Sigma_{X_n}^{-1}b_n ,\\
& = & c^2 b_n^T\Sigma_{X_n}^{-1}b_n ,\\
& = & c^2 b_n^T(\Sigma_{E_n} + b_nb_n^T)^{-1}b_n, \\
&  \stackrel{lemma (6)}{=}& c^2 b_n^T(\Sigma_{E_n}^{-1} - \frac{\Sigma_{E_n}^{-1}b_nb_n^T\Sigma_{E_n}^{-1}}{1+b_n^T\Sigma_{E_n}^{-1}b_n})b_n\\
& = & c^2\frac{b_n^T\Sigma_{E_n}^{-1}b_n}{1+b_n^T\Sigma_{E_n}^{-1}b_n}, \\
& = & c^2\frac{M(\mu_{\Sigma_{E_n}^{-1},b_n}) }{1+M(\mu_{\Sigma_{E_n}^{-1},b_n})}.
 \end{eqnarray*}
 Now we analyze the second term.  Apply lemma 6 for expressing inverse  again. 
 \begin{eqnarray*}
(\Sigma_{E_n}+b_nb_n^T)^{-1}b_n  & = & \Sigma_{E_n}^{-1}b_n - \frac{\Sigma_{E_n}^{-1}b_nb_n^T\Sigma_{E_n}^{-1}b_n}{1+b_n^T\Sigma_{E_n}^{-1}b_n} = \frac{\Sigma_{E_n}^{-1}b_n}{1+b_n^T\Sigma_{E_n}^{-1}b_n} , \\
\|c\Sigma_{X_n}^{-1}b_n\|^2 & = & c^2\frac{b_n^T\Sigma_{E_n}^{-2}b_n}{(1+b_n^T\Sigma_{E_n}^{-1}b_n)^2} = c^2 \frac{M(\mu_{\Sigma_{E_n}^{-2},b_n})}{(1+M(\mu_{\Sigma_{E_n}^{-1},b_n}))^2}.
 \end{eqnarray*}
 Then we can get the final form of the deviation measurement as
 \begin{eqnarray}
   \lim_{n \to \infty} \mathcal{D}(\tilde{a}_n) & = &  |\lim_{n \to \infty}M(\mu_{\Sigma_{X_n},c\Sigma_{X_n}^{-1}b_n}) - \lim_{n \to \infty}\|c\Sigma_{X_n}^{-1}b_n\|^2 \lim_{n \to \infty}M( \mu_{\Sigma_{X_n}}^\tau)|,\\\nonumber
   & = & | c^2\lim_{n \to \infty}\frac{M(\mu_{\Sigma_{E_n}^{-1},b_n}) }{1+M(\mu_{\Sigma_{E_n}^{-1},b_n})} - c^2\lim_{n \to \infty}\frac{M(\mu_{\Sigma_{E_n}^{-2},b_n})}{(1+M(\mu_{\Sigma_{E_n}^{-1},b_n}))^2}\lim_{n \to \infty}M(\mu_{\Sigma_{X_n}}^\tau) | ,\\
   & = & c^2 \lim_{n \to \infty} |  \frac{M(\mu_{\Sigma_{E_n}^{-1},b_n}) + M(\mu_{\Sigma_{E_n}^{-1},b_n})^2 - M(\mu_{\Sigma_{E_n}^{-2},b_n})M(\mu_{\Sigma_{X_n}}^\tau)}{(1+M(\mu_{\Sigma_{E_n}^{-1},b_n}))^2} |.
     \end{eqnarray}
     Using lemma 5, we get
     \begin{eqnarray*}
   \lim_{n \to \infty} M(\mu_{\Sigma_{E_n}^{-1},b_n})  & =&  r_b^2 M(\mu_{-1 }^\infty) ,\\
     \lim_{n \to \infty} M(\mu_{\Sigma_{E_n}^{-1},b_n})^2 & =&  r_b^4 M(\mu_{-1 }^\infty)^2 ,\\
    \lim_{n \to \infty} M(\mu_{\Sigma_{E_n}^{-2},b_n})M(\mu_{\Sigma_{X_n}}^\tau) &= &r_b^2M(\mu_{-2 }^\infty)(M(\mu_{1 }^\infty) + \tau_{\infty}(r_b^2)).
     \end{eqnarray*}
     Finally we get
   $$   \mathcal{D}(\tilde{a}_\infty) =  \frac{c^2r_b^2}{(1+r_b^2 M(\mu_{-1 }^\infty))^2}|M(\mu_{-1 }^\infty) - M(\mu_{-2 }^\infty)M(\mu_{1 }^\infty)+r_b^2M(\mu_{-1 }^\infty)^2-\tau_{\infty}(r_b^2)M(\mu_{-2}^\infty)|.$$

  \end{proof}

    \begin{remark}
    We put down some discussions of the proof.
    \begin{enumerate}
    \item  In the derivation process, we  use the fact multiple times that the limit of summation and multiplications can be calculated separately. This, known as algebraic limit theorem, applies because we assume that all limits exist by assumption 2 and 3. 
    \item  One can see from the (36) that, the   $M(\mu_{\Sigma_{X_n},a_n})$ coincides with $\|a_n\|^2M(\mu_{\Sigma_{X_n}}^\tau)$ and they play no roles in the final asymptotic form of the deviation.  The first moment deviation is determined by: how the first moment of spectral measure induced by confounding part differs from that of a uniform reference measure. 
    \end{enumerate}
    
    \end{remark}
     Now we already have its asymptotic value. Since in causal cases it is 0,  the condition that the confounder is not identifiable by this method is  that the deviation in confounding case is still 0 here. From (35), the deviation measurement heavily depends on the asymptotic spectral measures of the noise. In the next section, we would study it thoroughly.

  \subsection{Identifiability of confounding}  
  In this section, we study the identifiability of the confounder using our method. The non-identifiable models are those with $\mathcal{D}(\tilde{a}_\infty )=0 $. It is related to the eigenvalues of the covariance matrix of the noise. To start with, we consider if one can determine the sign of  the absolute part. We also assume the eigenvalues of covariance matrix of the noise are $\infty > \sigma_1 \geq \sigma_2 \geq ... \geq \sigma_{\infty} $. 
  \subsubsection*{Impossibility of universal identifiability}
  One ideal situation is that we can universally determine the sign of the absolute part regardless of the model parameters. In this way, $\mathcal{D}(\tilde{a}_\infty )$ might be always non-zero. We then show that this is not possible.  We consider the following decomposition.
  $$   \mathcal{D}(\tilde{a}_\infty) =  \frac{c^2r_b^2}{(1+r_b^2 M(\mu_{-1 }^\infty))^2}|\underbrace{M(\mu_{-1 }^\infty) - M(\mu_{-2 }^\infty)M(\mu_{1 }^\infty)}_{\text{part 1}}+\underbrace{r_b^2M(\mu_{-1 }^\infty)^2-\tau_{\infty}(r_b^2)M(\mu_{-2}^\infty)}_{\text{part 2}}|.$$
   Now we analyze the properties of the two parts.
       \begin{enumerate}
       \item $ M(\mu_{-1 }^\infty) - M(\mu_{-2 }^\infty)M(\mu_{1 }^\infty) \leq 0$. This is because
       \begin{eqnarray*}
     M(\mu_{\Sigma_{E_n}^{-1 }}^{\tau}) & = &  \tau_n(\Sigma_{E_n}^{-1}),\\
        & = & \frac{1}{n}\sum_{i=1}^{n} \sigma_i^{-1},\\
       & = & \frac{1}{n}\sum_{i=1}^{n} \sigma_i^{-2} \sigma_i, \\
       & \leq & \frac{1}{n}\sum_{i=1}^{n} \sigma_i^{-2} \frac{1}{n}\sum_{i=1}^{n} \sigma_i,\\
       & = &\tau_n(\Sigma_{E_n}^{-2})\tau_n(\Sigma_{E_n}),\\
       & = & M(\mu_{\Sigma_{E_n}^{-2 }}^{\tau})M(\mu_{\Sigma_{E_n}}^{\tau}),
          \end{eqnarray*}
          by Chebyshev's sum inequality, and then we have
          $$ M(\mu_{-1 }^\infty) = \lim_{n \to \infty} M(\mu_{\Sigma_{E_n}^{-1 }}^{\tau}) \leq \lim_{n \to \infty} M(\mu_{\Sigma_{E_n}^{-2 }}^{\tau})M(\mu_{\Sigma_{E_n}}^{\tau}) = M(\mu_{-2 }^\infty)M(\mu_{1 }^\infty)$$
          by order limit theorem.
          \item $ r_b^2M(\mu_{-1 }^\infty)^2-\tau_{\infty}(r_b^2)M(\mu_{-2}^\infty) \geq 0$. This is because
          \begin{eqnarray*}
          M(\mu_{\Sigma_{E_n}^{-1 }}^{\tau})^2 & = & \tau_n(\Sigma_{E_n}^{-1})^2\\
  & = &  (\frac{1}{n}\sum_{i=1}^{n} \sigma_i^{-1})^2,\\
               & = &  \frac{1}{n^2}(\sum_{i=1}^{n} \sigma_i^{-1})^2, \\
               & > &  \frac{1}{n^2} \sum_{i=1}^{n} \sigma_i^{-2}, \\
               & = &  \frac{1}{n}\tau_n(\Sigma_{E_n}^{-2}), \\
               & = & \frac{1}{n} M(\mu_{\Sigma_{E_n}^{-2 }}^{\tau}).
                  \end{eqnarray*}
       \end{enumerate}
      We have
      $$ M(\mu_{-1 }^\infty)^2 = \lim_{n \to \infty} M(\mu_{\Sigma_{E_n}^{-1 }}^\tau )^2 \geq  \lim_{n \to \infty} \frac{1}{n} M(\mu_{\Sigma_{E_n}^{-2 }}^{\tau}) =  \tau_{\infty}(1)M(\mu_{-2}^\infty) $$ 
   Thus, it is not possible to give an ``always identifiable'' conclusion, since summation of the two parts could be either bigger or smaller than 0, depending on $r_b$. This means we could possibly meet non-identifiable models.   In the next section, we would study the non-identifiable conditions. 
  \subsubsection*{General non-identifiable condition}
  Without any assumptions on the distribution of the eigenvalues, we here give a general non-identifiable condition. The confounder is not identifiable for models with $\mathcal{D}(\tilde{a}_\infty )$ being 0. This would lead to
  $$\frac{ M(\mu_{-1 }^\infty)}{1+r_b^2 M(\mu_{-1 }^\infty)} - \frac{ M(\mu_{-2 }^\infty) M(\mu_{1 }^\infty)}{(1+r_b^2 M(\mu_{-1 }^\infty))^2}-\frac{\tau_{\infty}(r_b^2)M(\mu_{-2}^\infty)}{(1+r_b^2 M(\mu_{-1 }^\infty))^2} = 0. $$
  Assuming boundness of the tracial moments as
  $$M(\mu_{-1 }^\infty) < \infty, $$
    we can proceed to analyze the equation as
  $$ M(\mu_{-1 }^\infty) - M(\mu_{-2 }^\infty) M(\mu_{1 }^\infty) +r_b^2M(\mu_{-1 }^\infty) ^2-\tau_{\infty}(r_b^2)M(\mu_{-2}^\infty)=0.$$
  We then express the $r_b$ using the quantity related to the moments.
  \begin{eqnarray}\nonumber
        r_b^2(M(\mu_{-1 }^\infty) ^2 -\tau_{\infty}(1)M(\mu_{-2}^\infty) ) & = & M(\mu_{-2}^\infty)M(\mu_{1}^\infty)-M(\mu_{-1}^\infty),\\
      r_b^2  & = & \frac{ M(\mu_{-2}^\infty)M(\mu_{1}^\infty)-M(\mu_{-1}^\infty) }{M(\mu_{-1 }^\infty) ^2 -\tau_{\infty}(1)M(\mu_{-2}^\infty)} .
       \end{eqnarray}
   If we then assume boundness of another tracial moments as
   $$M(\mu_{-2 }^\infty) < \infty, $$
   it can be further  refined as
   \begin{equation}
   r_b^2   =  \frac{ M(\mu_{-2}^\infty)M(\mu_{1}^\infty)-M(\mu_{-1}^\infty) }{M(\mu_{-1 }^\infty) ^2} = \frac{ M(\mu_{-2}^\infty)M(\mu_{1}^\infty)}{M(\mu_{-1 }^\infty)^2} - \frac{1}{M(\mu_{-1 }^\infty)}.
   \end{equation}  
    Since $r_b$ is the radius of the vector describing the confounding effect, and $M(\mu_i^\infty),(i\in \{-2,-1,1\})$ is the asymptotic moment of the tracial spectral measure of the noise, they are generated independently in nature. Thus, one should not expect that the  $r_b$ can be described by the moment of noise, in such a sophisticated way. This condition may be rarely satisfied, unless  one encounters very special models.   
     
     The above condition is a general one without any assumptions on the eigenvalue distribution of $\Sigma_{E_n}$. One may also be interested in the $\mathcal{D}(\tilde{a}_n ) $ when the eigenvalues follow some typical distributions.  We then consider 3  typical distributions of eigenvalues: constant, polynomial decay and exponential decay.  
    \subsubsection*{Eigenvalues are constant}
    We first consider  the eigenvalues being constant $\sigma_1$. It covers some special cases like  $\Sigma_{E_n} = \sigma_1 I_n$  ($I_n$ is the $n \times n$ identity matrix). It is obvious that in this case all moments are bounded, and we can check whether the general non-identifiable condition could be satisfied. 
    $$  r_b^2   =  \frac{ M(\mu_{-2}^\infty)M(\mu_{1}^\infty)}{M(\mu_{-1 }^\infty)^2} - \frac{1}{M(\mu_{-1 }^\infty)} =  \frac{1}{\sigma_1^{-1}}  - \frac{1}{\sigma_1^{-1}}= 0.$$
    The non-identifiable condition is that the $r_b$ is 0 here. It is never satisfied in confounding cases, where $r_b > 0$. Another way is to directly compute the $\mathcal{D}(\tilde{a}_n )$ when $n$ approaches infinity.
\begin{eqnarray*}
   \lim_{n \to \infty} \mathcal{D}(\tilde{a}_n ) & = & c^2r_b^2|\frac{ M(\mu_{-1 }^\infty)}{1+r_b^2 M(\mu_{-1 }^\infty)} - \frac{ M(\mu_{-2 }^\infty) M(\mu_{1 }^\infty)}{(1+r_b^2 M(\mu_{-1 }^\infty))^2}-\frac{\tau_{\infty}(r_b^2)M(\mu_{-2}^\infty)}{(1+r_b^2 M(\mu_{-1 }^\infty))^2}|,\\
    & = & c^2r_b^2|\frac{\sigma_1^{-1}}{1+\sigma_1^{-1}r_b^2}- \frac{\sigma_1^{-1}}{(1+\sigma_1^{-1}r_b^2)^2} - \frac{\sigma_1^{-2}\tau_{\infty}(r_b^2)}{(1+\sigma_1^{-1}r_b^2)^2} |, \\
    & = & c^2r_b^2|\frac{\sigma_1^{-1} + \sigma_1^{-2}r_b^2}{(1+\sigma_1^{-1}r_b^2)^2}- \frac{\sigma_1^{-1}}{(1+\sigma_1^{-1}r_b^2)^2}  |, \\
    & = &  c^2r_b^2\frac{ \sigma_1^{-2}r_b^2}{(1+\sigma_1^{-1}r_b^2)^2}, \\
    & = &  \frac{c^2r_b^4\sigma_1^{-2}}{(1+\sigma_1^{-1}r_b^2)^2}.
      \end{eqnarray*}
     Thus we have
     $$  \mathcal{D}(\tilde{a}_\infty ) =  \frac{c^2r_b^4}{\sigma_1^2+2\sigma_1r_b^2 + r_b^4}.$$
       It is clear that  the confounder is always identifiable here.
      \subsubsection*{Eigenvalues decay polynomially}  
             We study the case where the eigenvalues decay polynomially as
             $$ \sigma_i = \sigma_1 i^{-1}.$$
             We write the traces of the covariance matrices here.
              \begin{eqnarray*}
              M(\mu_{\Sigma_{E_n}}^\tau)=\tau_n(\Sigma_{E_n}) & = & \frac{\sigma_1}{n} \sum_{i=1}^{n}i^{-1} = \frac{\sigma_1}{n}(\ln n + \kappa +  \epsilon(n)).\\
              M(\mu_{\Sigma_{E_n}^{-1}}^\tau)=\tau_n(\Sigma_{E_n}^{-1}) & = & \frac{\sigma_1^{-1}}{n}\sum_{i=1}^{n}i = \frac{\sigma_1^{-1}}{2}(n+1) .\\
        M(\mu_{\Sigma_{E_n}^{-2}}^\tau)= \tau_n(\Sigma_{E_n}^{-2})    & = & \frac{\sigma_1^{-2}}{n}\sum_{i=1}^{n}i^2 = \frac{\sigma_1^{-2}}{6}(n+1)(2n+1).
          \end{eqnarray*}
    Here we write the Harmonic series  as
          \begin{eqnarray*}
         \sum_{i=1}^{n}i^{-1} & = & \ln n + \kappa +  \epsilon(n), \\
          \epsilon(n) & = & O(\frac{1}{n}).
          \end{eqnarray*}
        $\kappa$ is  Euler Mascheroni constant.   Then we proceed to make analysis. 
          \begin{eqnarray*}
  \lim_{n \to \infty} \mathcal{D}(\tilde{a}_n ) & = & c^2r_b^2|\underbrace{\frac{ M(\mu_{-1 }^\infty)}{1+r_b^2 M(\mu_{-1 }^\infty)}}_{\theta_1^\infty} - \underbrace{\frac{ M(\mu_{-2 }^\infty) M(\mu_{1 }^\infty)}{(1+r_b^2 M(\mu_{-1 }^\infty))^2}}_{\theta_2^\infty}-\underbrace{\frac{\tau_{\infty}(r_b^2)M(\mu_{-2}^\infty)}{(1+r_b^2 M(\mu_{-1 }^\infty))^2}}_{\theta_3^\infty}|,
  \end{eqnarray*}
  Here we include 3 thetas for ease of representation. 
\begin{eqnarray*}
\theta_1^\infty   & = & \lim_{n \to \infty}\frac{\tau_n(\Sigma_{E_n}^{-1})}{1+r_b^2\tau_n(\Sigma_{E_n}^{-1})},\\
\theta_2^\infty   & = &      \lim_{n \to \infty}\frac{\tau_n(\Sigma_{E_n}^{-2})\tau_n(\Sigma_{E_n})}{(1+r_b^2\tau_n(\Sigma_{E_n}^{-1}))^2},\\
\theta_3^\infty  & = &  \lim_{n \to \infty}\frac{r_b^2\tau_n(\Sigma_{E_n}^{-2})}{n(1+r_b^2\tau_n(\Sigma_{E_n}^{-1}))^2}.
\end{eqnarray*}
         \begin{eqnarray*}
         \lim_{n \to \infty}\frac{\tau_n(\Sigma_{E_n}^{-2})\tau_n(\Sigma_{E_n})}{\tau_n(\Sigma_{E_n}^{-1})^2} & = & \lim_{n \to \infty} \frac{2\sigma_1(\ln n + \kappa +  \epsilon(n))}{3n}\frac{(n+1)(2n+1)}{(n+1)^2} = 0, \\
          \lim_{n \to \infty} \frac{\tau_n(\Sigma_{E_n}^{-2})}{n\tau_n(\Sigma_{E_n}^{-1})^2}  & = &    \lim_{n \to \infty} \frac{2(n+1)(2n+1)}{3n(n+1)^2} = 0.\\
  \theta_1^\infty    & = &  \frac{1}{r_b^2},\\
                 \theta_2^\infty & = &    \lim_{n \to \infty}\frac{\tau_n(\Sigma_{E_n}^{-2})\tau_n(\Sigma_{E_n})}{(1+r_b^2\tau_n(\Sigma_{E_n}^{-1}))^2}  = 0.\\
        \theta_3^\infty & = &   \lim_{n \to \infty} \frac{r_b^2\tau_n(\Sigma_{E_n}^{-2})}{n(1+r_b^2\tau_n(\Sigma_{E_n}^{-1}))^2} = 0. 
   \end{eqnarray*}       
   Then we have 
   $$  \mathcal{D}(\tilde{a}_\infty )  = c^2r_b^2|\frac{1}{r_b^2} - 0 - 0| = c^2.$$

    So  asymptotically, we still have the condition that the $ \mathcal{D}(\tilde{a}_n )>0$. Thus, we claim that the confounder is  identifiable by our method in the polynomial decay cases.
    \begin{remark} We want to put down some comments here.
    \begin{enumerate}
   \item Since the $\sigma_1$ is assumed to be bounded (a constant, for example), the largest eigenvalue of $\Sigma_{E_n}^{-1}$ would be unbounded in the limit case.  We are still doing analysis based on the convergence results of lemma 5. The justification should come from the postulate 1 in paper \citep{janzing2017detecting} with boundness assumption dropped. Then we could make analysis because we know exactly how the eigenvalue grows, and thus the support of $\mu_{-1}^\infty$. 
   \item Although the moments $M(\mu_{-1}^\infty)$ and $M(\mu_{-2}^\infty)$  do not exist in the limit case, the ratios $\theta_1^\infty$ and $\theta_2^\infty+\theta_3^\infty$ do exist. To understand what they represent, recall the equation (36).
    \begin{eqnarray*}
   r_b^2 \theta_1^\infty & = &   \lim_{n \to \infty}M(\mu_{\Sigma_{X_n},\Sigma_{X_n}^{-1}b_n}), \\
   r_b^2(\theta_2^\infty+\theta_3^\infty) & = &  \lim_{n \to \infty}\|\Sigma_{X_n}^{-1}b_n\|^2 M( \mu_{\Sigma_{X_n}}^\tau).
    \end{eqnarray*}
    Note that we canceled the effect of the scalar $c$ in the equation. Thus, we have
    \begin{enumerate}
    \item  $\theta_1^\infty < \infty$ directly follows the existence of first moment of the spectral measure induced by the confounding part.
    \item $\theta_2^\infty +\theta_3^\infty < \infty$ directly follows the existence of first moment of the normalized tracial spectral measure enlarged by the norm of the confounding part.
    \end{enumerate}
    \end{enumerate}
    \end{remark}
  \subsubsection*{Eigenvalues decay exponentially} 
           We study the case where the eigenvalues decay exponentially as
         $$ \sigma_i = \sigma_1 e^{-(i-1)}.$$
 We analyze the traces here.
   \begin{eqnarray*}
               \tau_n(\Sigma_{E_n}) & = & \frac{\sigma_1}{n} \sum_{i=1}^{n}e^{-(i-1)} = \frac{\sigma_1(1-e^{-n})}{n(1-e^{-1})}.\\
               \tau_n(\Sigma_{E_n}^{-1}) & = & \frac{\sigma_1^{-1}}{n}\sum_{i=1}^{n}e^{(i-1)} = \frac{\sigma_1^{-1}(1-e^{n})}{n(1-e)} .\\
           \tau_n(\Sigma_{E_n}^{-2})    & = & \frac{\sigma_1^{-2}}{n}\sum_{i=1}^{n}e^{(2i-2)}  = \frac{\sigma_1^{-2}(1-e^{2n})}{n(1-e^2)}.
           \end{eqnarray*}
           We then analyze the thetas here.
   \begin{eqnarray*}
   \lim_{n \to \infty} \frac{\tau_n(\Sigma_{E_n}^{-2})\tau_n(\Sigma_{E_n}) }{\tau_n(\Sigma_{E_n}^{-1})^2}
      & = &    \lim_{n \to \infty} \frac{(e^{2n}-1)(e - e^{-n+1})}{(1-e^n)^2(1+e)}\sigma_1 \\
      & = &   \lim_{n \to \infty} \frac{(e^{n}+1)(e - e^{-n+1})}{(e^n-1)(e+1)}\sigma_1 , \\
      & = & \lim_{n \to \infty} \frac{e^{n+1}-e^{-n+1} }{e^{n+1}+e^n - e - 1}\sigma_1\\
      &  = & \frac{\sigma_1}{1 + e^{-1}}\\
      \lim_{n \to \infty} \frac{\tau_n(\Sigma_{E_n}^{-2})}{n\tau_n(\Sigma_{E_n}^{-1})^2}  & = &    \lim_{n \to \infty} \frac{(1+e^n)(1-e)}{(1-e^n)(1+e)} = \frac{e-1}{e+1}.\\
        \theta_1^\infty    & = &  \frac{1}{r_b^2},\\
              \theta_2^\infty & = &   \lim_{n \to \infty} \frac{\tau_n(\Sigma_{E_n}^{-2})\tau_n(\Sigma_{E_n})}{(1+r_b^2\tau_n(\Sigma_{E_n}^{-1}))^2}  = \frac{\sigma_1}{(1 + e^{-1})r_b^4},\\
       \theta_3^\infty & = &   \lim_{n \to \infty} \frac{r_b^2\tau_n(\Sigma_{E_n}^{-2})}{n(1+r_b^2\tau_n(\Sigma_{E_n}^{-1}))^2} = \frac{e-1}{(e+1)r_b^2}. 
                    \end{eqnarray*} 
 Then we have
  $$ \mathcal{D}(\tilde{a}_\infty )  = c^2r_b^2|\frac{1}{r_b^2} - \frac{\sigma_1}{(1 + e^{-1})r_b^4} - \frac{e-1}{(e+1)r_b^2}| = c^2|\frac{2}{e+1} - \frac{e\sigma_1}{( e + 1)r_b^2}|. $$
 If we let it be 0, we get 
 $$r_b^2 = \frac{e\sigma_1}{2}.$$
      Now we have the non-identifiable condition at hand. Clearly the confounder is not always identifiable. When the  $r_b^2$ is $0.5e\sigma_1$, in the confounding case we get an asymptotic 0 of the deviation measurement. However, we do not consider this as a ``often happen'' case: the $b_n$ is independently generated with respect to $\Sigma_{E_n}$, and $b_n$ should rarely be with a norm that aligns with the largest eigenvalue of $\Sigma_{E_n}$, although we do not exclude the possibility of non-identifiable models.
      
       After we study those cases, we get back to the general analysis. In fact, the deviation measurement of the first moment is asymptotically 0 in purely causal cases. It already reaches the lower bound of absolute values. Thus, no matter what the covariance matrix of the noise is, we are still able to claim that the deviation measurement in confounding cases is not less than that in non-confounding cases. The ``not less" condition is, in most of the situations, enough for our method to work. By analyzing the non-identifiable condition, we gain more confidence: when the eigenvalues of covariance matrix of the noise follow some typical distributions, the confounder is always identifiable. In other cases, it is almost identifiable, since the non-identifiable condition is hard to satisfy. In the next section, we would describe the method and discuss the empirical estimations.
             \section{Methodology}
       
              In this section, we present our method.  The practical deviation measurement is based on the data observed from joint distribution of  $(X_n,Y)$. We would now show how to estimate the empirical deviation.  We first summarize our algorithm here.
             \begin{algorithm}[h]
             \KwIn{I.i.d observations of $(Y,X_n)$ where $X_n=(x_1,...,x_n)^T$, threshold $\gamma$; }
            \KwOut{The decision;}
             Compute the empirical covariance matrices $\widehat{\Sigma}_{YX_n}$ and $\widehat{\Sigma}_{X_n}$\;
              Compute $\widehat{\mathcal{D}}(\tilde{a}_n)  = |\widehat{\Sigma}_{YX_n}\widehat{\Sigma}_{X_n}^{-1}\widehat{\Sigma}_{X_nY}  - \|\widehat{\Sigma}_{X_n}^{-1}\widehat{\Sigma}_{X_nY} \|^2\tau_n(\widehat{\Sigma}_{X_n})|$\;
              \uIf{$\widehat{\mathcal{D}}(\tilde{a}_n)\leq \gamma$}{
              Output ``No confounder''\;  }
              \Else{
              Output ``There is a confounder''\;}
               \caption{Confounder detection by measuring first moment deviation}
              \end{algorithm}
              Our method is based on a threshold of the empirical deviations. If it is close enough to 0, we treat it as non-confounding cases; else we report the existence of a confounder.  The inclusion of the threshold $\gamma$ is not only because of the estimation error caused by finite sampling size. Since the justifications are made in infinite dimensions, practical moment coincidence should also be interpreted with a certain loose sense, as approximate coincidence. This is realized by including a threshold $\gamma$ when checking the deviation measurement being ``0 or not''. Here we also claim that our method can consistently estimate the true value of $\mathcal{D}(\tilde{a}_n)$. To understand this, let the sample size be $L$. We rewrite the $\mathcal{D}(\tilde{a}_n)$ using the estimators as
                 \begin{eqnarray}
           \widehat{\mathcal{D}}(\tilde{a}_n)  & = & |\widehat{\Sigma}_{YX_n}\widehat{\Sigma}_{X_n}^{-1}\widehat{\Sigma}_{X_nY}  - \|\widehat{\Sigma}_{X_n}^{-1}\widehat{\Sigma}_{X_nY} \|^2\tau_n(\widehat{\Sigma}_{X_n})|
             \end{eqnarray}
            The hatted symbols are standard estimators of the covariance matrices. The general covariance matrix estimations are known to be  consistent with a rate $L^{-\frac{1}{2}}$. This guarantees the consistency of our estimator,  since accurate estimation of the covariance matrices leads to accurate deviation measurement.

             In the next section, we would conduct various experiments to test the performance of our method.

            \section{Experiments}    
            
            In this section, we would test the proposed confounding detection method in various aspects. Each subtitle describes the focus of respective subsections.  If not specified,  all model related variables are drawn from standard normal distributions (either one or multi dimensional).  We record the  distribution of $\mathcal{D}(\tilde{a}_n)$ (or $\widehat{\mathcal{D}}(\tilde{a}_n)$) based on  200 runs.   $n$ is the dimensionality. Note that $\mathcal{D}(\tilde{a}_n)$ here is the deviation computed using the true model parameters and $\Sigma_{X_n}$, and $\widehat{\mathcal{D}}(\tilde{a}_n)$ is the deviation computed from the samples, as showed in algorithm 1. What we mean ``distribution" later in the figures is the empirical probability of $\mathcal{D}(\tilde{a}_n)$ (or $\widehat{\mathcal{D}}(\tilde{a}_n)$) exceeding certain value, calculated as: number of experiments with $\mathcal{D}(\tilde{a}_n)$ greater or equal to a value divided by the total number of experiments. We here distinguish causal and confounding models of  
             \begin{eqnarray*}
             X_n & = & b_nZ+E_n,\\
             Y & = & a_n^TX_n + cZ+ F,
              \end{eqnarray*}
                 by setting $c=0$ and $c$ being non-zero, instead of changing $b_n$.
            \subsection{First moment deviation and dimensionality}
            All the justifications of the asymptotic forms are done in infinite dimensional models. However, practical data is  with finite dimensions. The moment deviation behavior in practical cases is important.   We first show the empirical first moment concentration  results for randomly generated $a_n$ and $\Sigma_{X_n}$.  For random matrix $\Sigma_{X_n} = \Sigma_{E_n} + b_nb_n^T$, we generate it using this method. First, we are using the same method as described in footnote 3 to generate $\Sigma_{E_n}$.  When sampling the entries of $\Gamma_n$, we use a uniform distribution defined on $[0.5,1]$. Then add $b_nb_n^T$  to get $\Sigma_{X_n}$. Here we consider the non-confounding cases by setting $c=0$ in the model described by (1) and (2), such that 
            $$ \tilde{a}_n = a_n,$$
            The probability for $\mathcal{D}(\tilde{a}_n)$ exceeding certain values are plotted in figure 2. Besides normal coefficients, we add a test generating the coefficients using a uniform distribution on $[-0.5,0.5]$. We also normalize the norm of $a_n$ and $b_n$ to 1. These results show that although the theoretic moment concentration  happens only when dimension is very high, the empirical results are much better.   The first moment of induced spectral measure almost coincides with that of a tracial measure  when we only have 10 dimensions. 
           
              \begin{figure}[h]%
              \centering
              \subfloat[Normal coefficients]{{\includegraphics[width=.5\textwidth,height =.4\textwidth]{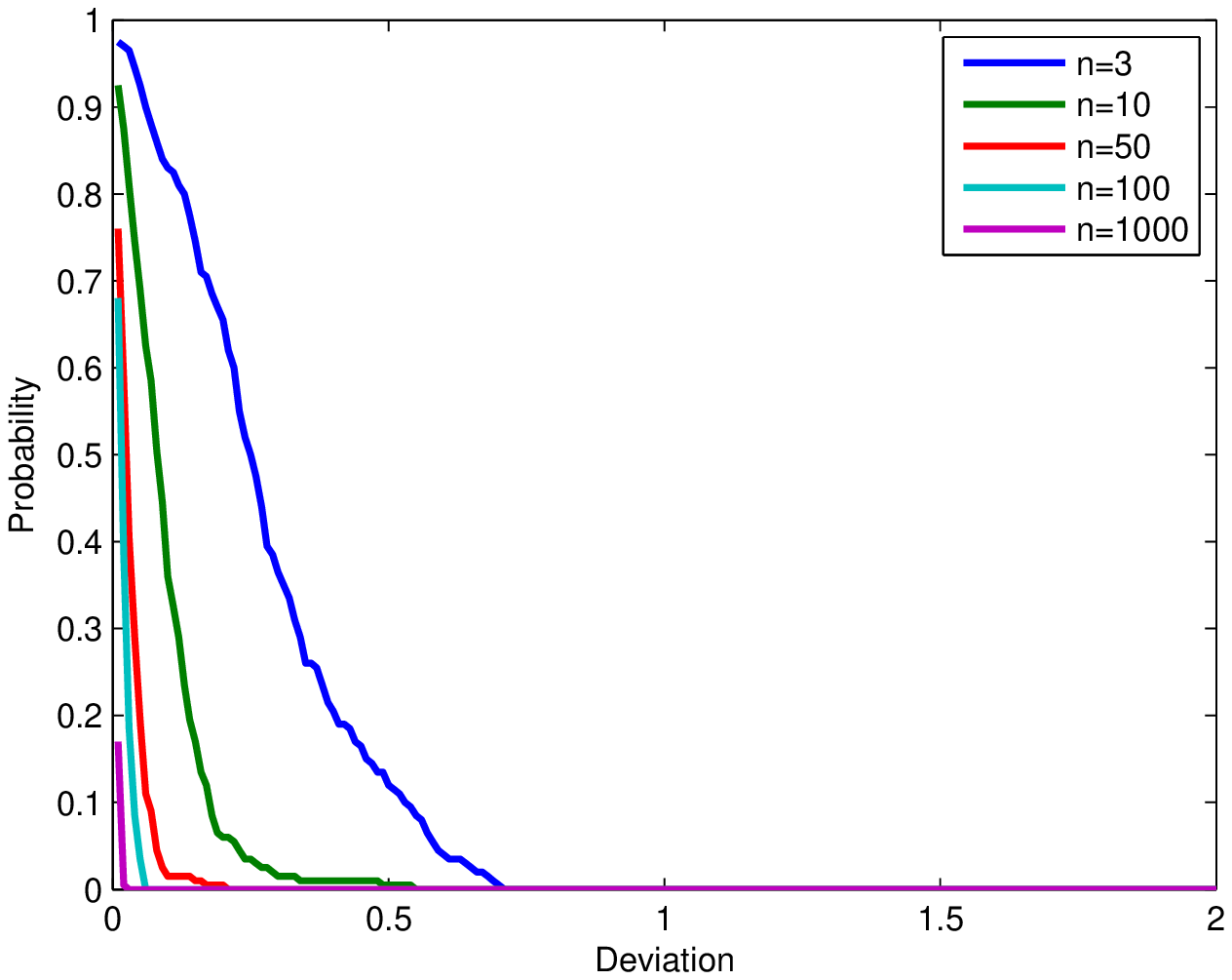} }}
       \subfloat[Uniform coefficients ]{{\includegraphics[width=.5\textwidth,height =.4\textwidth ]{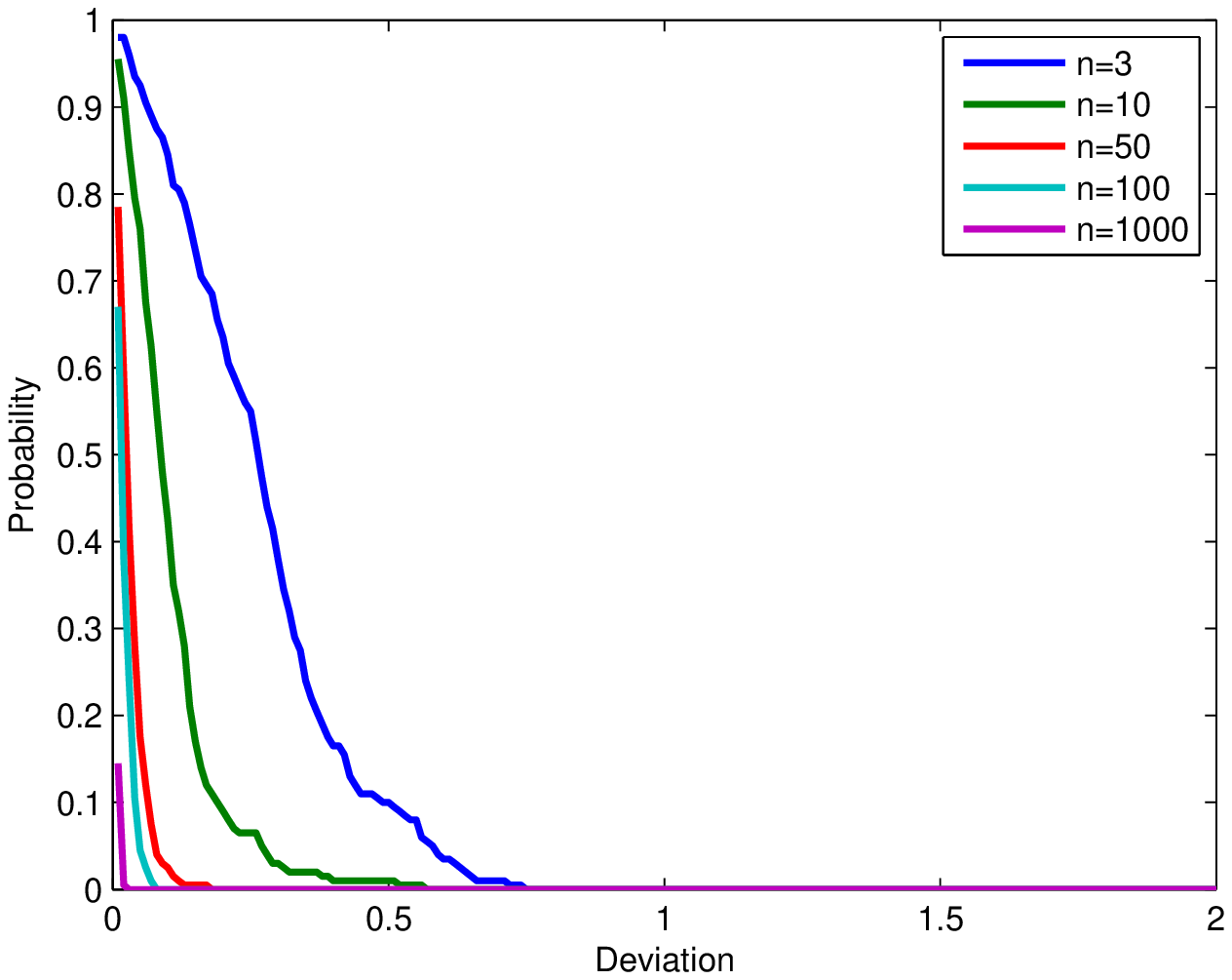} }} 
               \caption{Distribution of  $\mathcal{D}(\tilde{a}_n)$ in causal cases}
               \end{figure}
        We then study the deviations when there is a confounder. In this case $$ \tilde{a}_n = a_n + c(\Sigma_{E_n}+b_nb_n^T)^{-1}b_n$$
        The whole term $\mathcal{D}(\tilde{a}_n)$ is larger than 0 as the previous theoretic results show.  The experimental results are in figure 3.  There are 2 observations that are in consistency with our previous  theoretic analysis.
          \begin{enumerate}
          \item When there is a confounder, $\mathcal{D}(\tilde{a}_n)$ is no longer 0.
          \item The larger the $c$ is, the larger the $\mathcal{D}(\tilde{a}_n)$  tends to be. A clear evidence is that when $c$ is uniform on $[2,3]$ rather than gaussian, the $\mathcal{D}(\tilde{a}_n)$  becomes larger in general. 
           \end{enumerate}
            \begin{figure}[h]%
           \subfloat[Distribution of $\mathcal{D}(\tilde{a}_n)$ ($c$ being normal) ]{{\includegraphics[width=.5\textwidth,height =.4\textwidth]{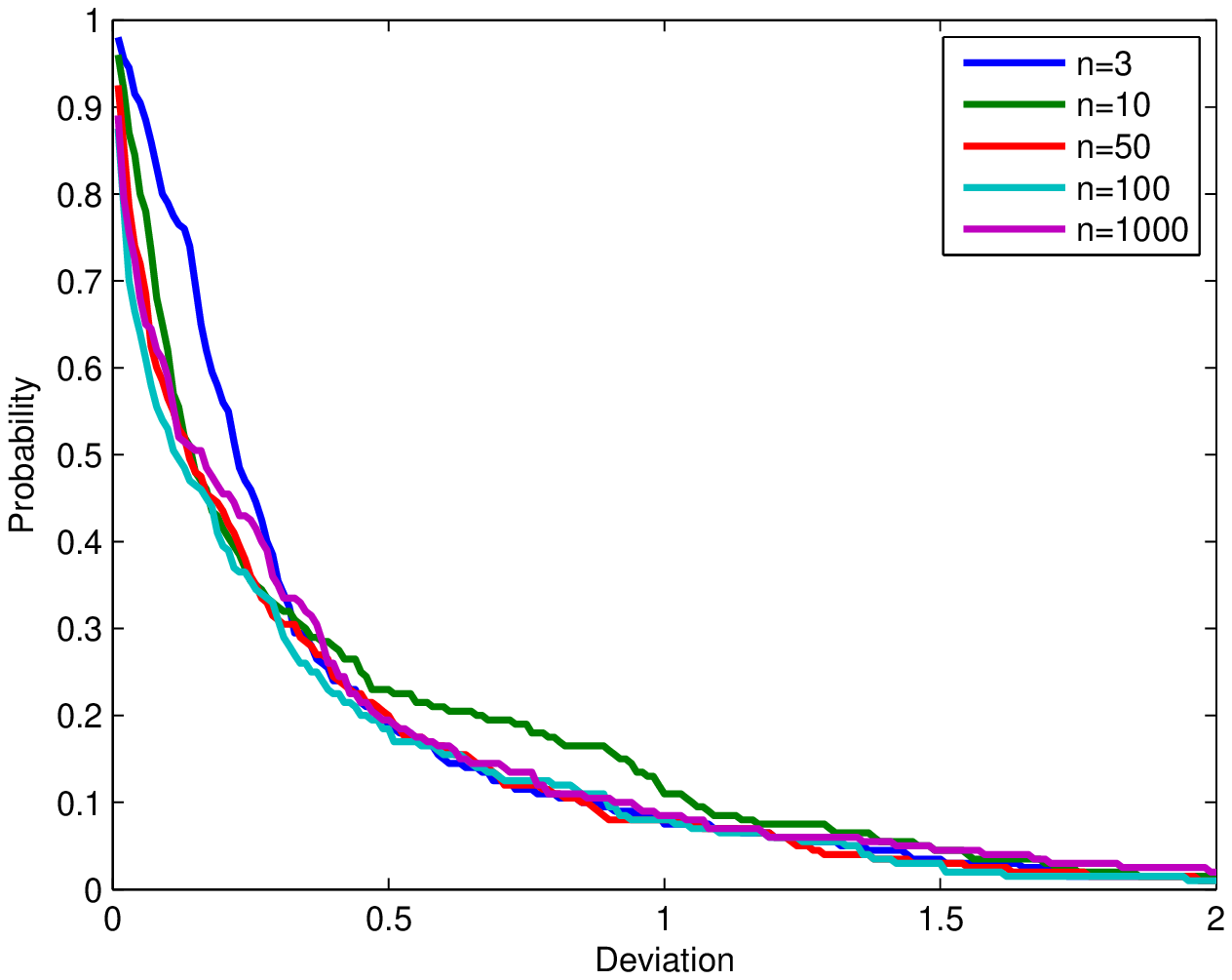} }}
            \subfloat[Distribution of $\mathcal{D}(\tilde{a}_n) ~(c \in {[2, 3]})$]{{\includegraphics[width=.5\textwidth,height =.4\textwidth]{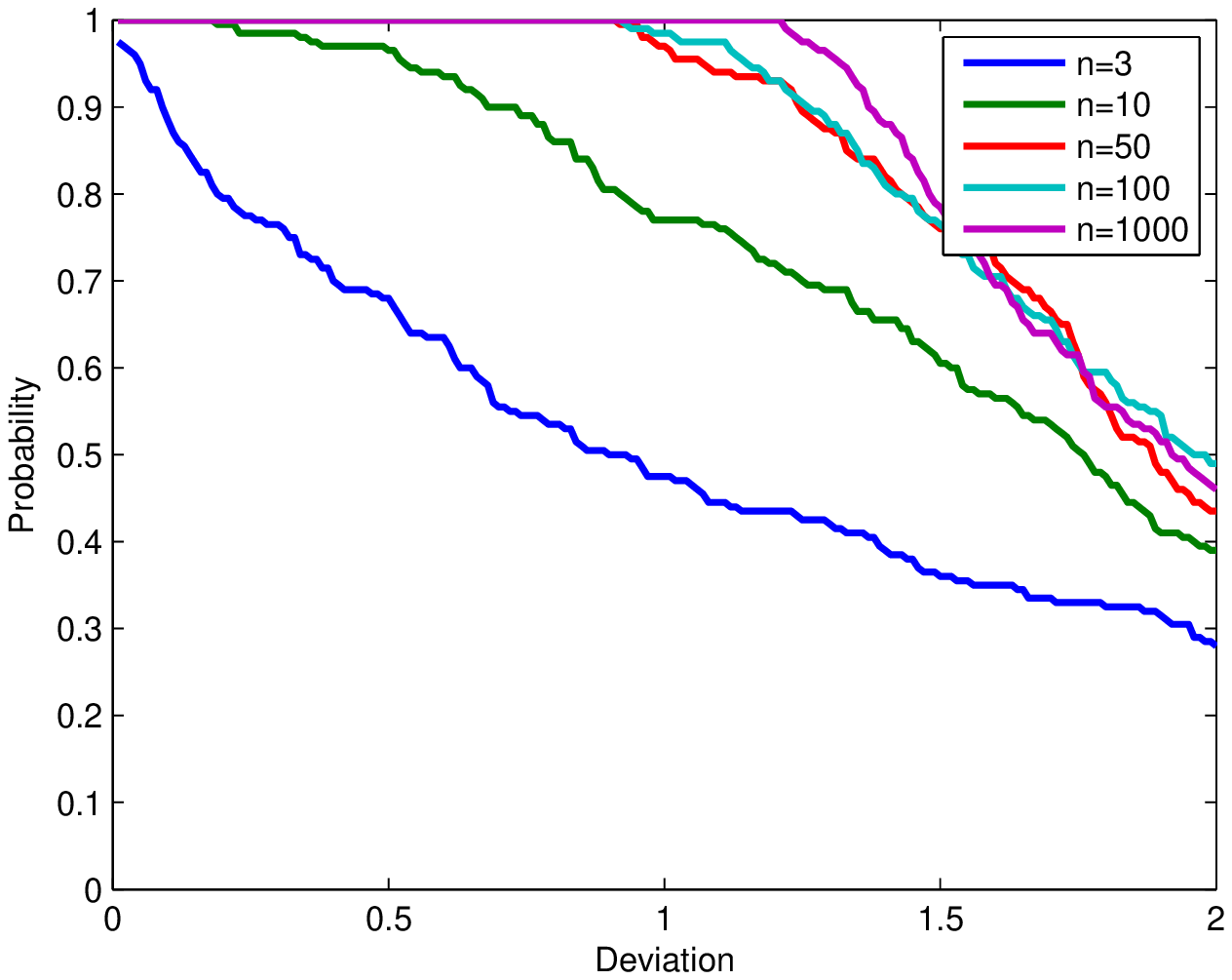} }}  
         \caption{Distribution of $\mathcal{D}(\tilde{a}_n)$ in confounding cases}
              \end{figure}   
   These experiments deliver important messages. It shows that the difference between     $\mathcal{D}(\tilde{a}_n)$ in confounding and non-confounding cases are obvious. In the confounding cases, first moment deviations are clear.  This indicates a behavior difference of the deviation measurement in confounding and non-confounding cases. The next thing is about the estimation from observations. Note that all the experiments here are  using the true model parameters $a_n,b_n,c$ and the true $\Sigma_{X_n}$.     In practice, we can only estimate these from observations. We will study this in the next section.    
   \subsection{Empirical estimations}
   In practice, we only have observational points from the model rather than the true model parameters. What if we estimate all the values using observations? We study the  $\widehat{\mathcal{D}}(\tilde{a}_n)$ estimated from the samples. In this section, we first generate model parameters, and then sample  from the model. Algorithm is applied on the observations. We first study the non-confounding cases where $$\tilde{a}_n = a_n.$$ The plots showing the curves are in figure 4. We also show the curves of $\widehat{\mathcal{D}}(\tilde{a}_n)$ when there is a confounder and  $$ \tilde{a}_n = a_n + c(\Sigma_{E_n}+b_nb_n^T)^{-1}b_n.$$
     The plots showing the curves are in figure 5 ((a), (b), (c) for $c$ being uniformly distributed on $[1,2]$, and (d), (e), (f) for $c$ being uniformly distributed on $[2,3]$).
     \begin{figure}[h]%
     \centering
     \subfloat[True values ]{{\includegraphics[width=.33\textwidth,height=.3\textwidth]{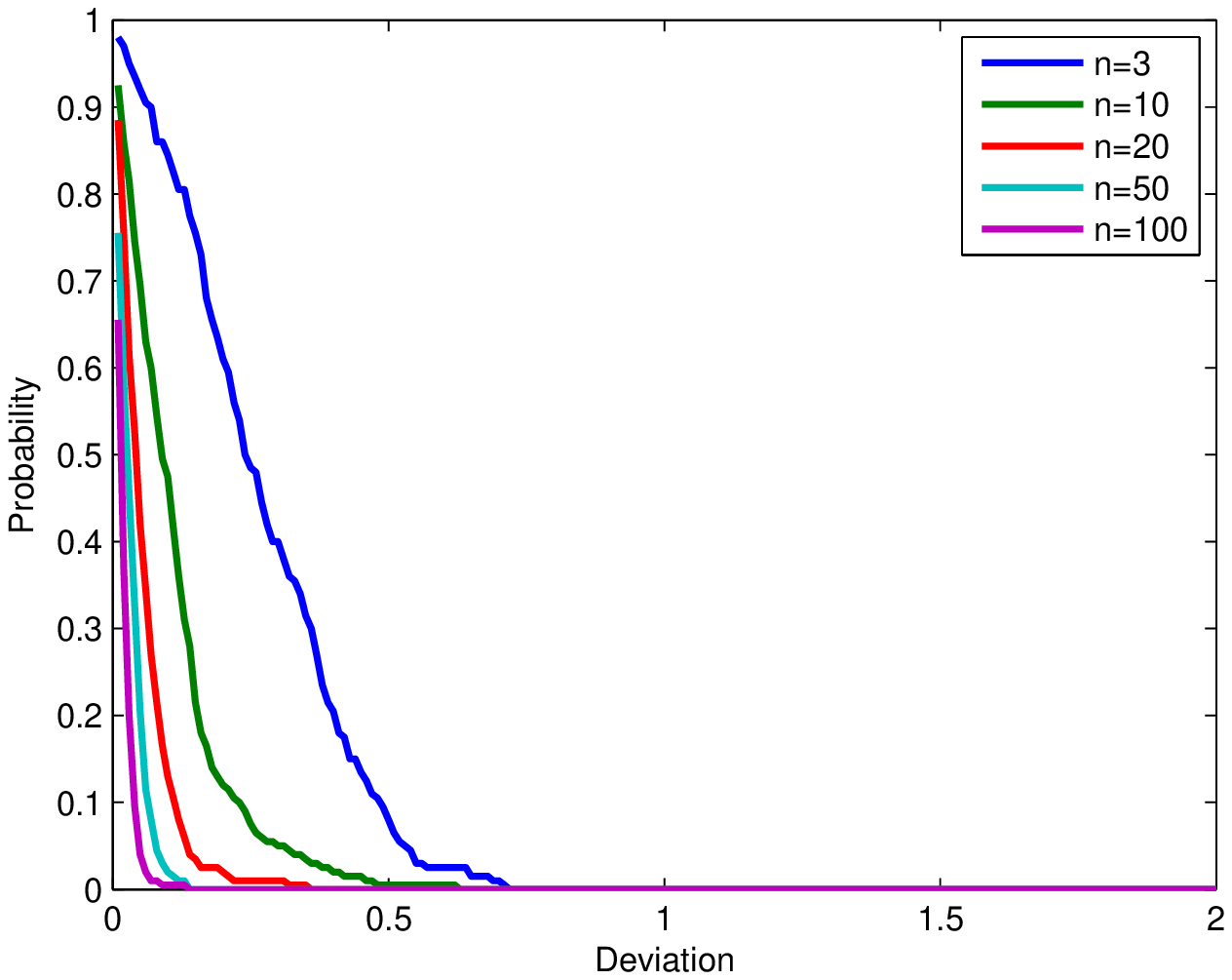} }}
     \subfloat[Sample size = 500 ]{{\includegraphics[width=.33\textwidth,height=.3\textwidth]{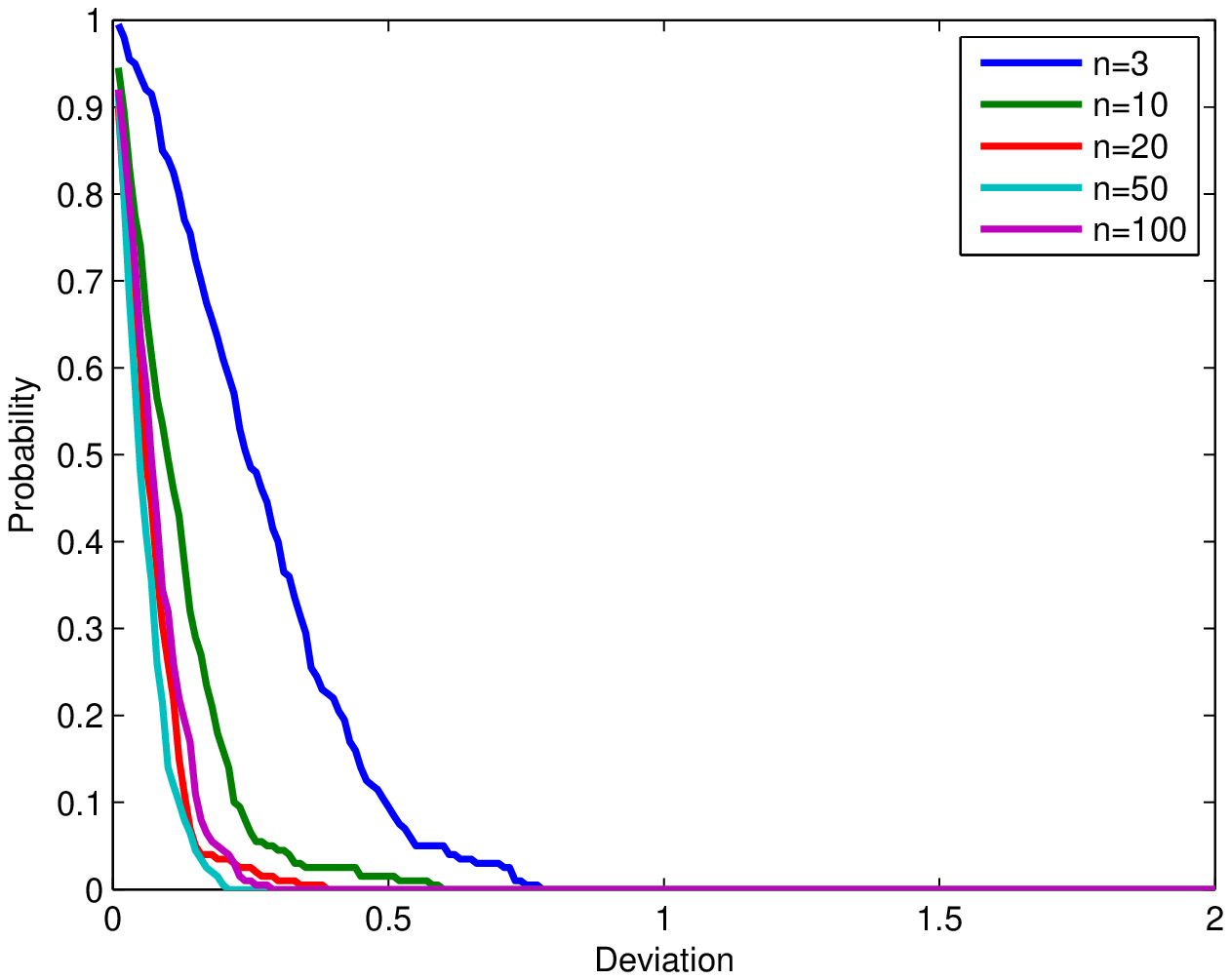} }}  
    \subfloat[Sample size = 5000 ]{{\includegraphics[width=.33\textwidth,height=.3\textwidth]{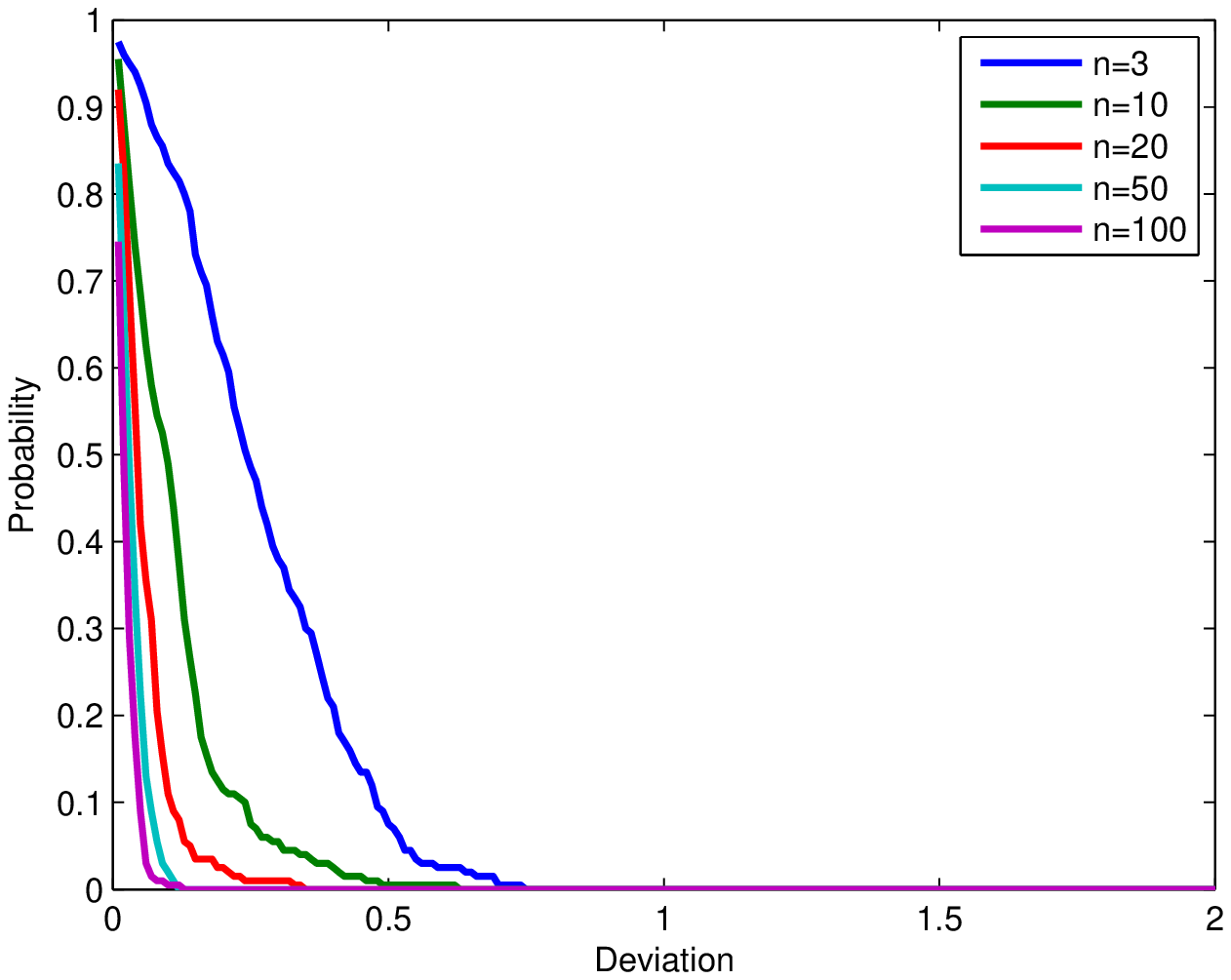} }}
      \caption{Distribution of $\widehat{\mathcal{D}}(\tilde{a}_n)$ in  causal cases}
       \end{figure}  
    \begin{figure}[h]%
       \centering
        \subfloat[True values ]{{\includegraphics[width=.33\textwidth,height=.3\textwidth]{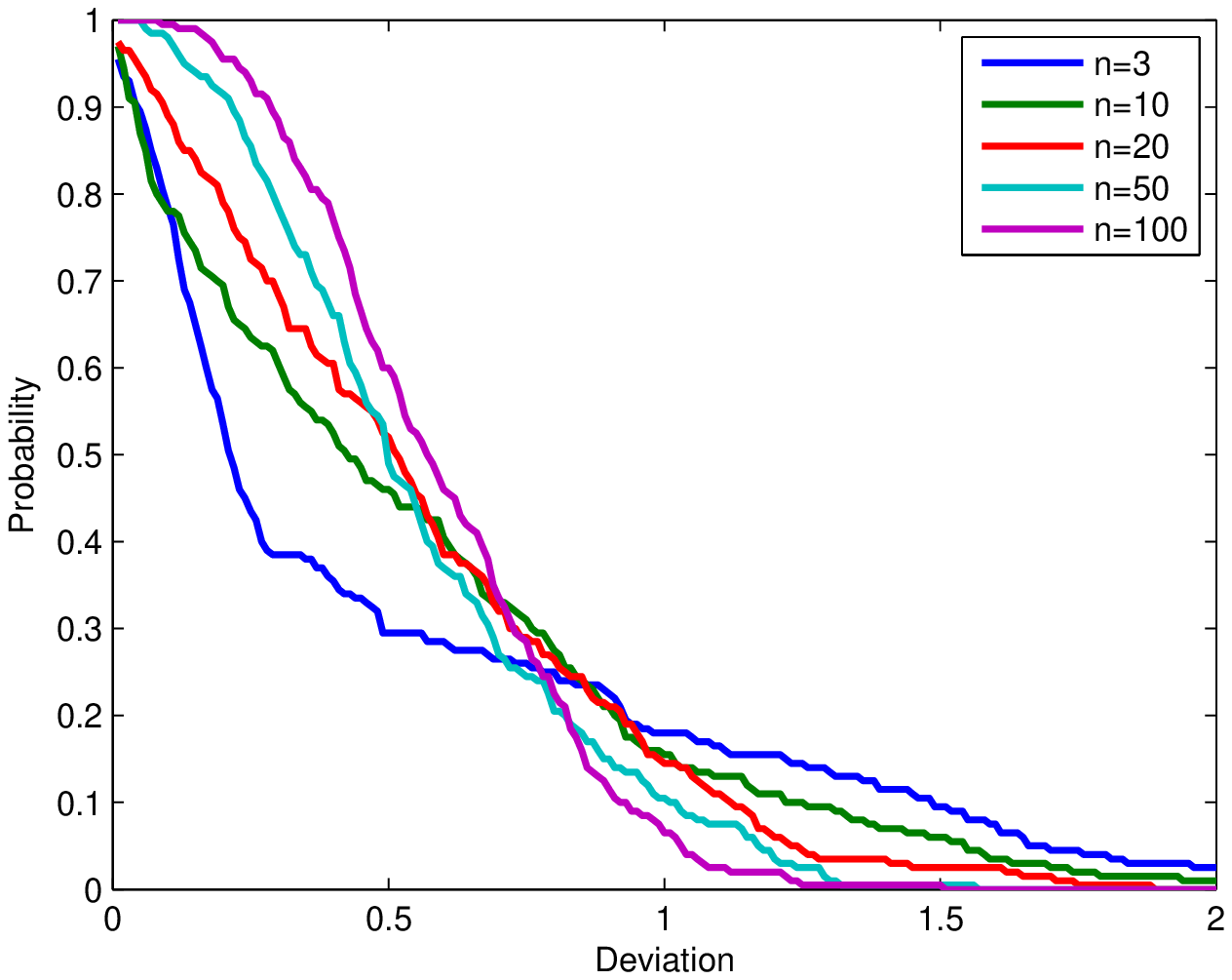} }}
       \subfloat[Sample size = 500 ]{{\includegraphics[width=.33\textwidth,height=.3\textwidth]{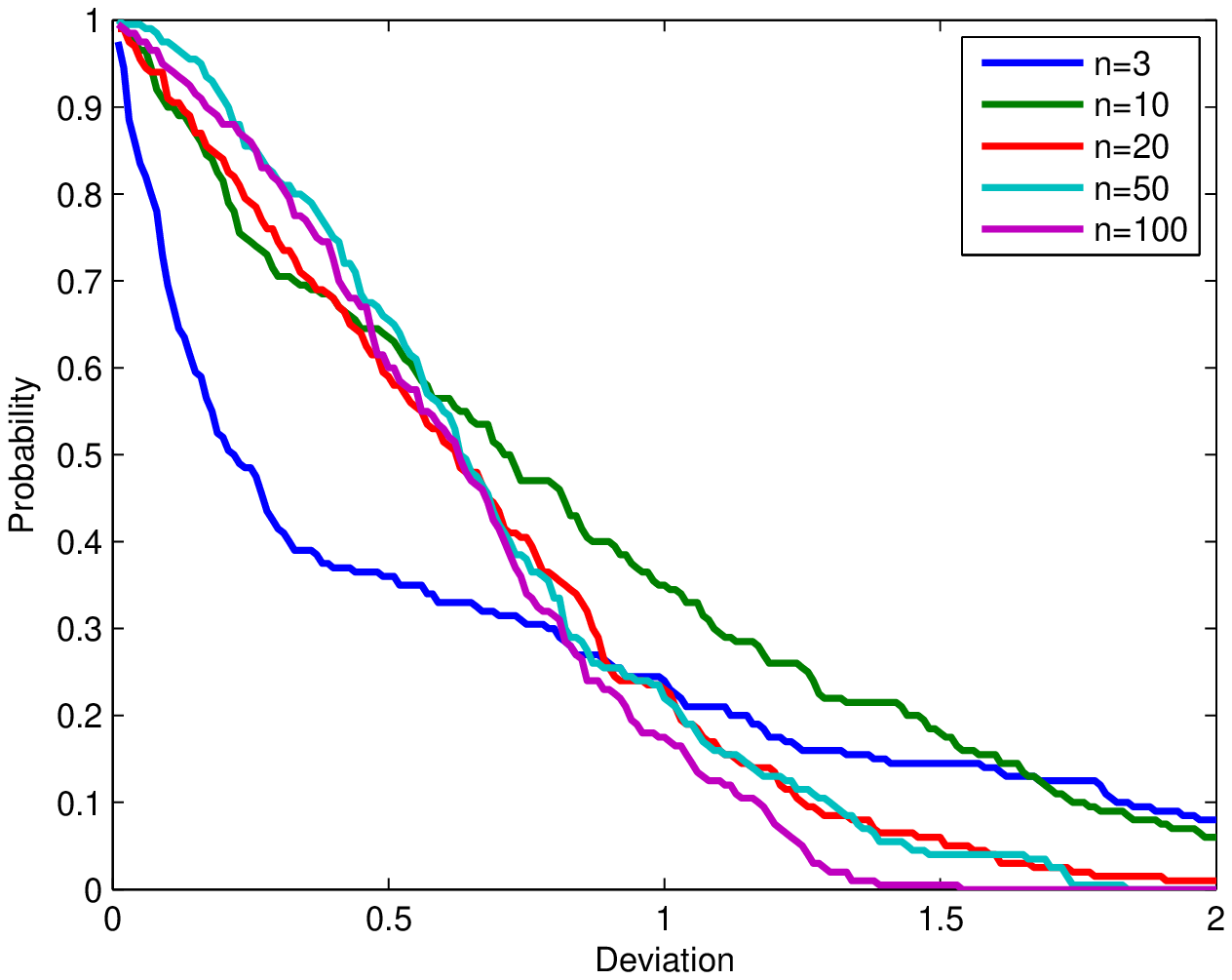} }}  
     \subfloat[Sample size = 5000 ]{{\includegraphics[width=.33\textwidth,height=.3\textwidth]{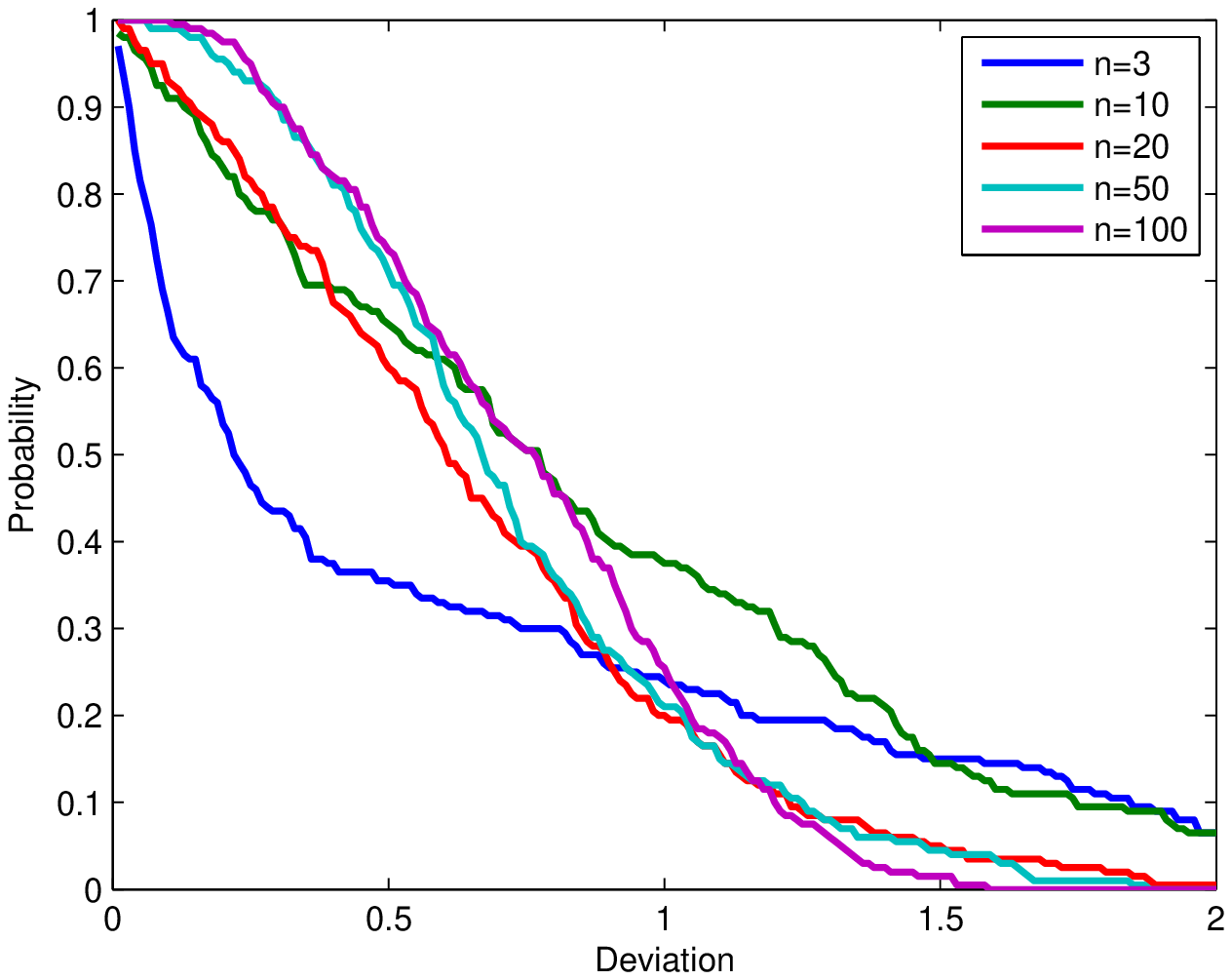} }}\quad
        \subfloat[True values ]{{\includegraphics[width=.33\textwidth,height=.3\textwidth]{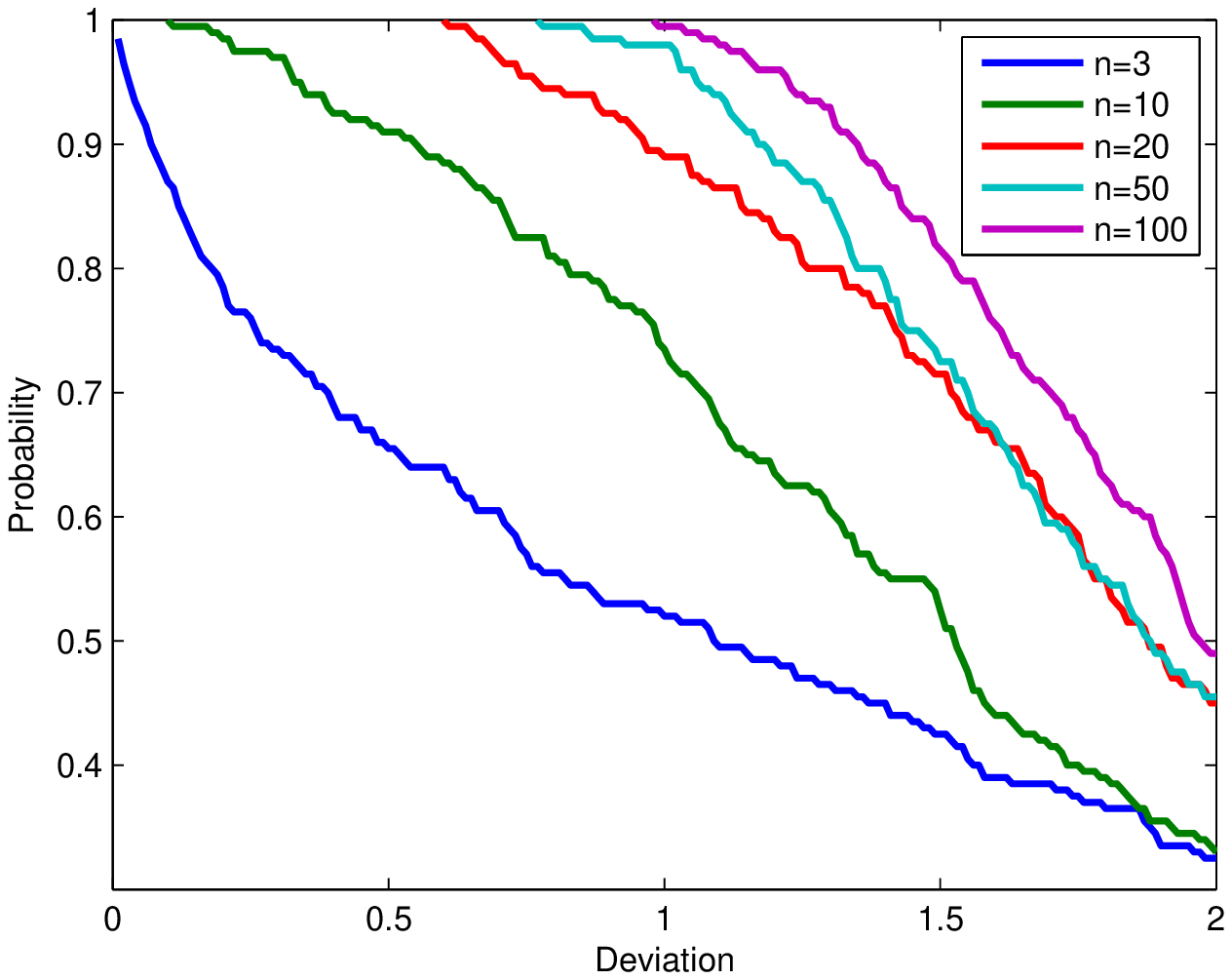} }}
       \subfloat[Sample size = 500 ]{{\includegraphics[width=.33\textwidth,height=.3\textwidth]{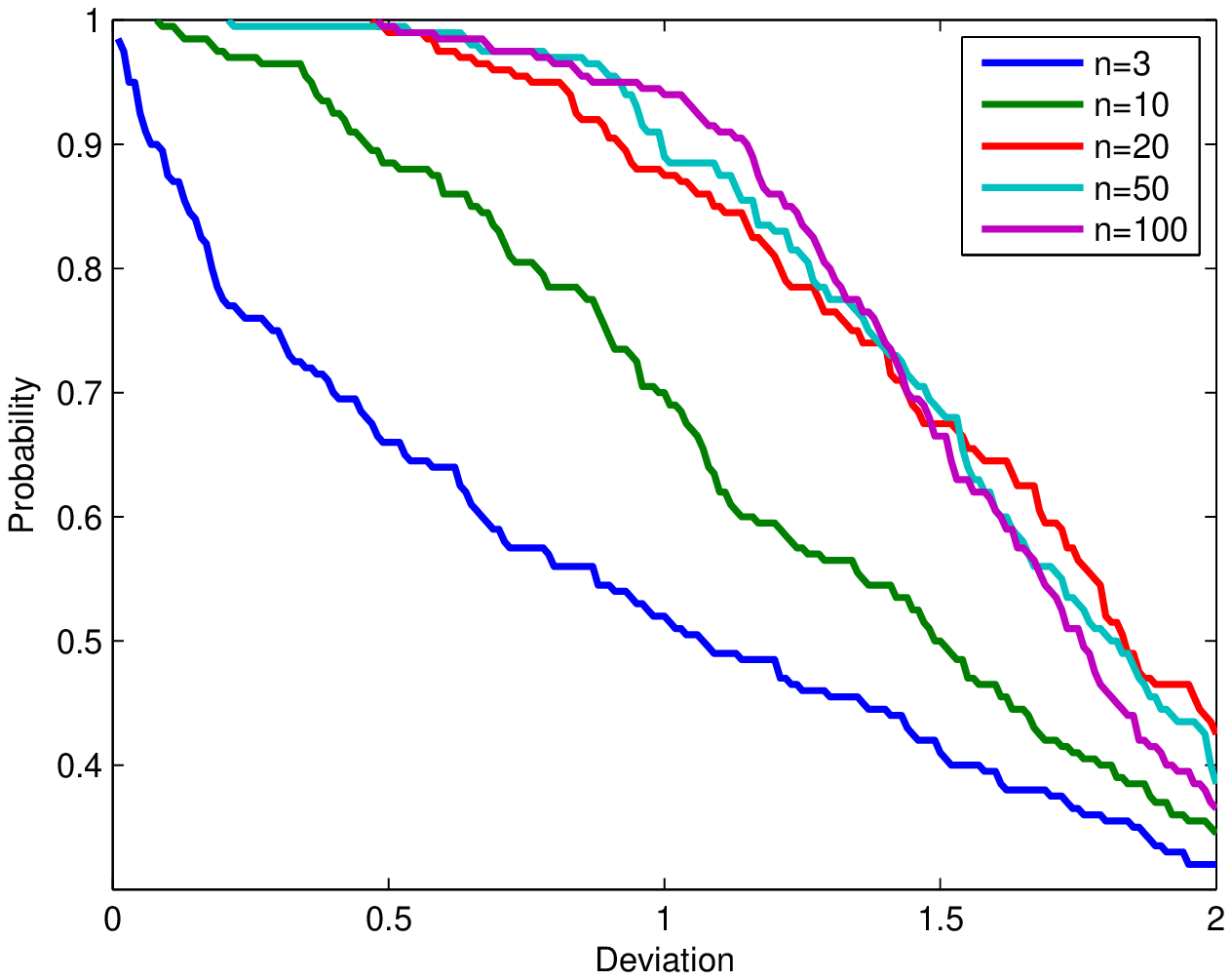} }}  
       \subfloat[Sample size = 5000 ]{{\includegraphics[width=.33\textwidth,height=.3\textwidth]{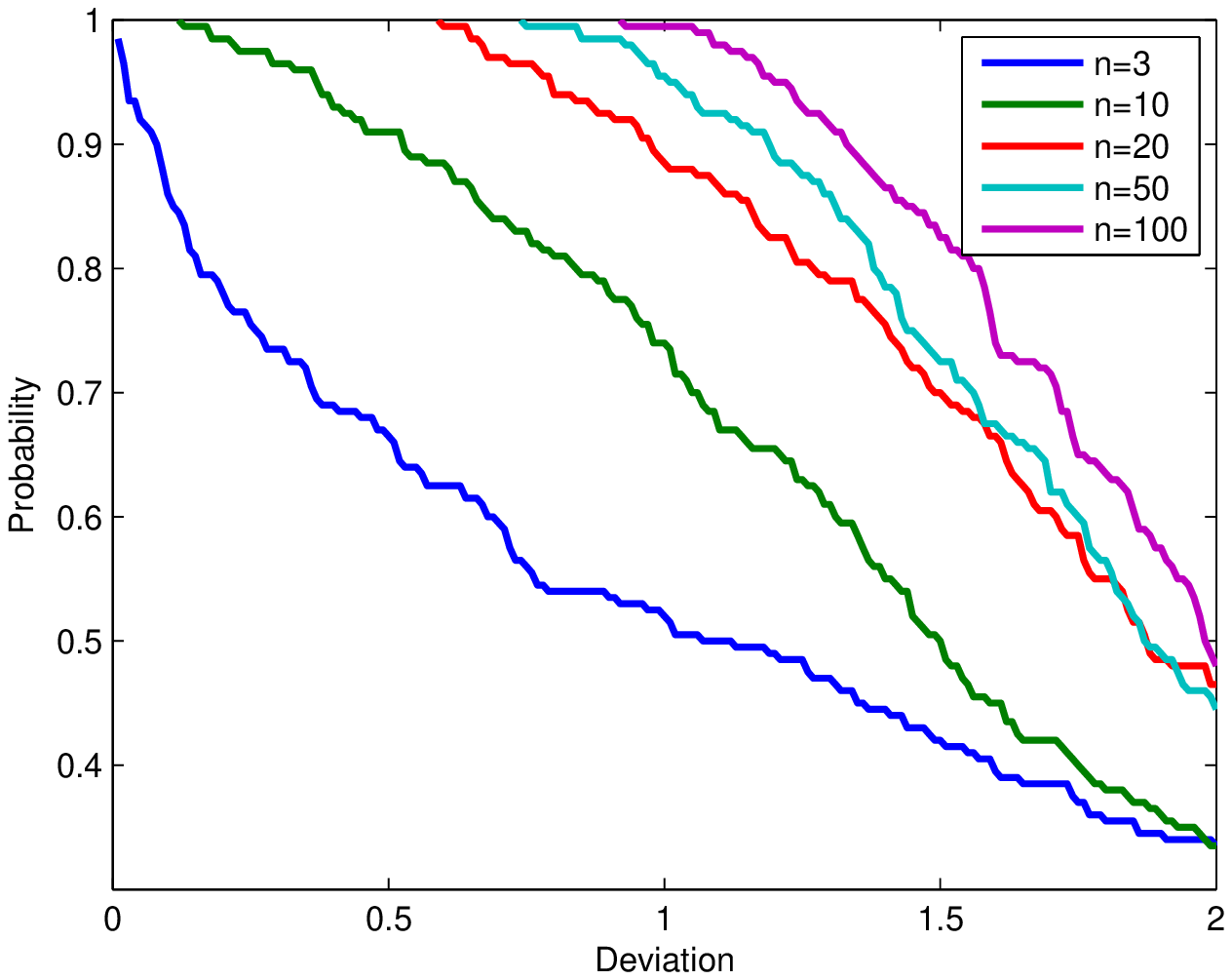} }}
       \caption{Distribution of $\widehat{\mathcal{D}}(\tilde{a}_n)$  in confounding cases}
       \end{figure} 
  From the figures, we can see that the  distribution of  $ \widehat{\mathcal{D}}(\tilde{a}_n)$  is similar to that of its true value. This may be due to the fact that we are able to consistently estimate the true $\mathcal{D}(\tilde{a}_n)$ from observations when sample size goes large. The empirical error of estimations seems to be acceptable.
                         
 Another observation is that when $c$ becomes larger, the confounding effects are more  obvious and the deviations tend to be larger. This matches with our previous analysis about the $\mathcal{D}(\tilde{a}_n)$, that it enlarges with the scalar $c$.
       \subsection{Comparative study}
       We compare our method with the method by \citep{janzing2017detecting}, which is mainly based on the estimation of $\beta^*$. We adopt the  default setting of the algorithm given in \citep{janzing2017detecting} (denoted as J \& S in tables). 	For our method, the threshold $\gamma$ for concluding a confounder is set to be 0.5. For the method J \& S, we report a confounder if $\beta^* > 0.5$. Sample size is fixed to be 500. $c$ is $0$ or uniform on $[2,3]$. Table 1 shows the results of applying algorithms on data, that are samples from models with   normal coefficients (entries of $a_n$ and $b_n$), and  uniform  coefficients (on $[-0.5,0.5]$). The noises are standard normally distributed. Table 2 shows the results on models with normal coefficients and  different noise distributions. Noise distributions are set to be: 1. Normal distribution. 2. Multidimensional student t distribution with degree of freedom 10. 3. Log normal distribution. 4. Mixture of two normal distribution, with equal probability and mean uniformly drawn from $[-0.5,0.5]$. Note that for multivariate distributions, we are feeding randomly generated covariance matrix (using the same method as that of footnote 3, with the entries of the diagonal matrix $\Gamma_n$ sampled from a uniform distribution on $(0,1)$).  Table 3 shows the results on models with normal noise, and eigenvalues of $\Sigma_{E_n}$   follow specified distributions. In the exponential decay cases, we use a rate $e^{-\frac{1}{5}}$ instead of $e^{-1}$ to avoid too fast decay. Based on the observations, we put down some discussions here.
         \begin{table}[t]
          \caption{Accuracy of algorithms (standard normal noise)}
             \centering
             \begin{tabular}{c|ccc|ccc}\hline
             & \multicolumn{3}{c|}{Normal coefficients}&\multicolumn{3}{c}{Uniform coefficients}\\
          &   $n$ = 10  &    $n$ = 20    &  $n$ = 30 & $n$ = 10  &    $n$ = 20    &   $n$ = 30 \\ \hline
           $c=0$     &&&&&& \\
                          Ours        &   98 \%   &    100\%    &  100\%   &   99 \%   &    99\%    &  100\% \\
          J \& S&   35 \%  &    31 \%  &  37 \%  & 28 \%  &    34 \%  &  35 \% \\ \hline  
                           $c\in [2,3]$ &&& \\ 
                             Ours        &   84 \% &    97\%    &  93\%  &   88 \% &    94\%    &  95\%  \\
            J \& S &   73 \%  &    75 \%  & 67 \%  &   79 \%  &    66 \%  &  68 \% \\ \hline 
          \end{tabular}
          \end{table}
              \begin{table}[t]
           \caption{Accuracy of algorithms (random $\Sigma_{E_n}$)}
        \centering
          \begin{tabular}{c|ccc|ccc}\hline 
          &  $n$ = 10  &    $n$ = 20    &   $n$= 30 & $n$ = 10  &    $n$ = 20    &   $n$ = 30
       \\\hline
        $c=0$ & \multicolumn{3}{c|}{Normal noise} &  \multicolumn{3}{c}{Student $t$ distributed noise} \\      
           Ours        &  97 \%   &    96\%    &  94\%  &  86 \%   &    89\%    &  93\% \\
          J \& S &   64 \%  &    88 \%  &  85 \%  & 77 \%  &    91 \%  &  96 \%   \\ \hline  
           $c\in [2,3]$ & & &  \\ 
       Ours        &   86 \% &    87\%    &  88\%   &  80 \%   &    87\%    &  90\% \\
       J \& S &   85 \%  &    95 \%  &  94 \%  & 74 \%  &    76 \%  &  82\%  \\ \hline 
     $c=0$      & \multicolumn{3}{c|}{Log normal noise}  &  \multicolumn{3}{c}{Mixture of two normal noise}\\ 
          Ours        &  88 \%   &    98\%    &  99\%  &  99 \%   &    99\%    &  100\% \\
        J \& S &   64 \%  &    75 \%  &  81 \%  & 39 \%  &    52\%  &  73 \%   \\ \hline  
  $c\in [2,3]$ & & &  \\ 
         Ours        &   72 \% &    58\%    &  58\%   &  89 \%   &    96\%    &  99\% \\
     J \& S &   62 \%  &    58 \%  &  60 \%  & 85 \%  &    84 \%  &  91\%  \\ \hline
    \end{tabular}
     \end{table}
        \begin{enumerate}
       \item In finite dimensional cases, when the  vector $a_n$ does not perfectly lie on the ``center position'' of the eigenspace of $\Sigma_{X_n}$, method by \citep{janzing2017detecting}  tends to include a “confounding part” because of the variational pattern of the spectral measure. 
      \item When noises are with random covariance matrices, the performance of our algorithm decreases in cases with log normal noises, but is still good in other cases. The distribution of the noise seems to have an impact on the results.   The performance of the method \citep{janzing2017detecting} is generally acceptable.
      \item Our method performs well  when eigenvalues of $\Sigma_{E_n}$ follow typical spectral decay patterns. This matches with our previous theoretic analysis, that the confounder is almost identifiable in these cases. However, an exception is found in exponential decay cases with $n=30$. This is because of the unstable estimation caused by too small eigenvalues. 
      \end{enumerate}
         \begin{table}[t]
         \caption{Accuracy of algorithms ($\Sigma_{E_n}$ with specified eigenvalue distributions)}
        \centering
        \begin{tabular}{c|ccc|ccc|ccc}\hline
   & \multicolumn{3}{c|}{Constant}&\multicolumn{3}{c|}{Polynomial decay} & \multicolumn{3}{c}{Exponential decay}\\
                &   $n$ = 10  &    $n$ = 20    &  $n$ = 30 & $n$ = 10  &    $n$ = 20    &   $n$ = 30 &   $n$ = 10  &    $n$ = 20    &  $n$ = 30 \\ \hline
                 $c=0$  &&&&&&   \\
                                Ours        &   99 \%   &    100\%    &  100\%   &   99 \%   &    100\%    &  100\% &   99 \%   &   99 \%    &  33\%\\
                J \& S&   38 \%  &    40 \%  &  42 \%  & 40 \%  &    56 \%  &  53 \%  & 36 \% & 45 \% & 12 \% \\ \hline  
               $c\in [2,3]$ &&&&&&  \\ 
                Ours        &   84 \% &    96\%    &  96\%  &   98 \% &    100\%    &  100\%  &   93 \%   &    93\%    &  88\%\\
                  J \& S &   77 \%  &   67 \%  &  66 \%  &   93 \%  &    95 \%  &  98 \%  &   75 \%  &    67 \%  &  65 \% \\ \hline 
                \end{tabular}
                \end{table}
      
        Note that in our experiment, we  actually compute the $\widehat{\mathcal{D}}(\tilde{a}_n)$ and conclude confounding when it exceeds a certain threshold. The choice of the threshold plays an important role.  One question of interest is that how large it should be. We would study this in the next section.

 \subsection{Threshold $\gamma$}
           Our algorithm concludes confounding based on the rule that the computed  $\widehat{\mathcal{D}}(\tilde{a}_n)$ exceeds certain threshold $\gamma$.  Different thresholds would lead to different errors. If the threshold is too small, we have a high false positive rate. But if it is too large, we would have a low true positive rate. To study this, we conduct some experiments. We use the  settings  of experiments showed in table 2 (data dimension 10 and 20) with normal noise, and vary the threshold from 0 to 1. For each threshold, we conduct 100 experiments on confounding cases and 100 experiments on non-confounding cases, and record the true positive and false positive rate. We plot the results in figure 6.
                            
         \begin{figure}[h]%
        \centering
        \subfloat[$n$ = 10]{{\includegraphics[width=.5\textwidth,height =.4\textwidth]{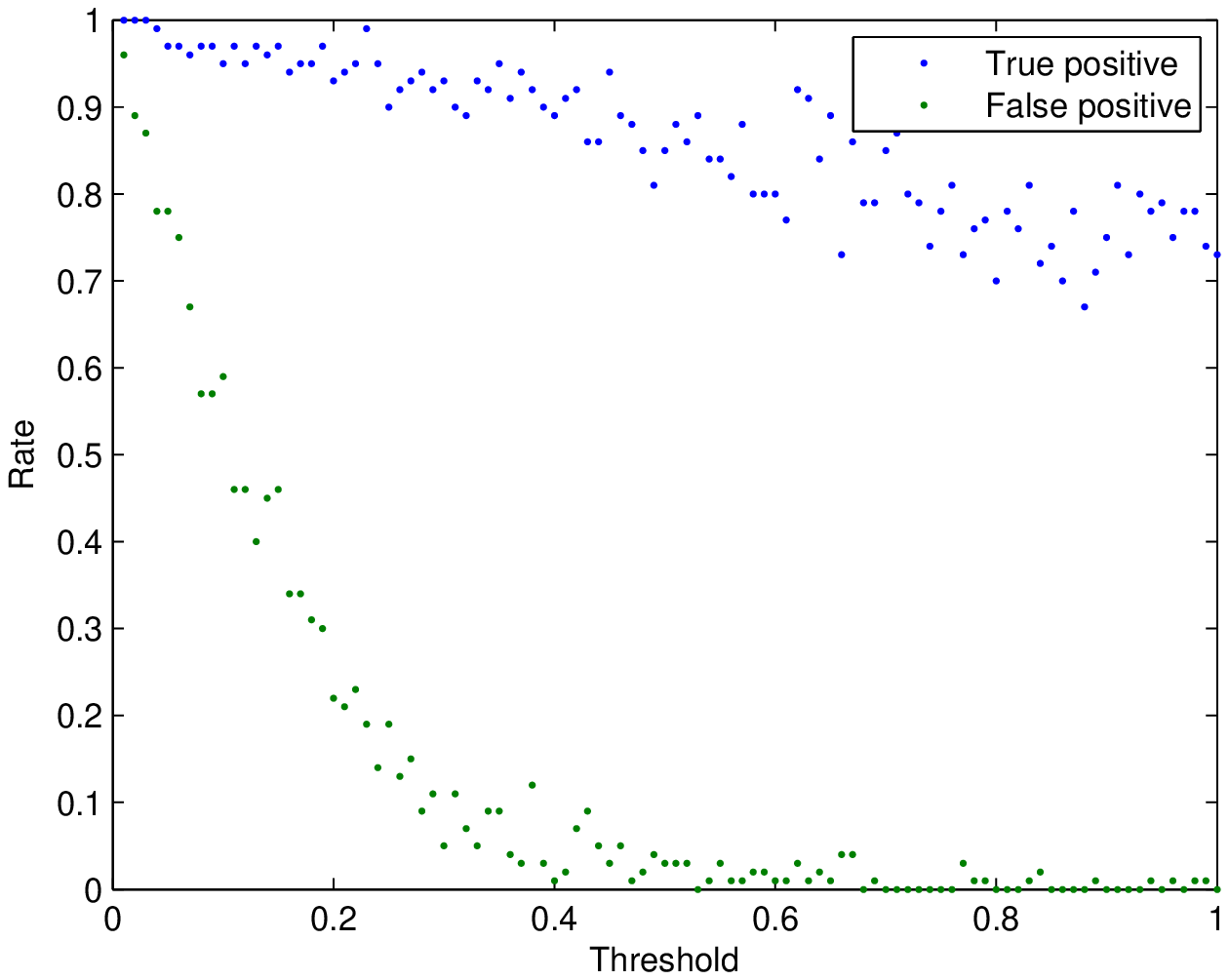} }}
         \subfloat[$n$ = 20 ]{{\includegraphics[width=.5\textwidth,height =.4\textwidth]{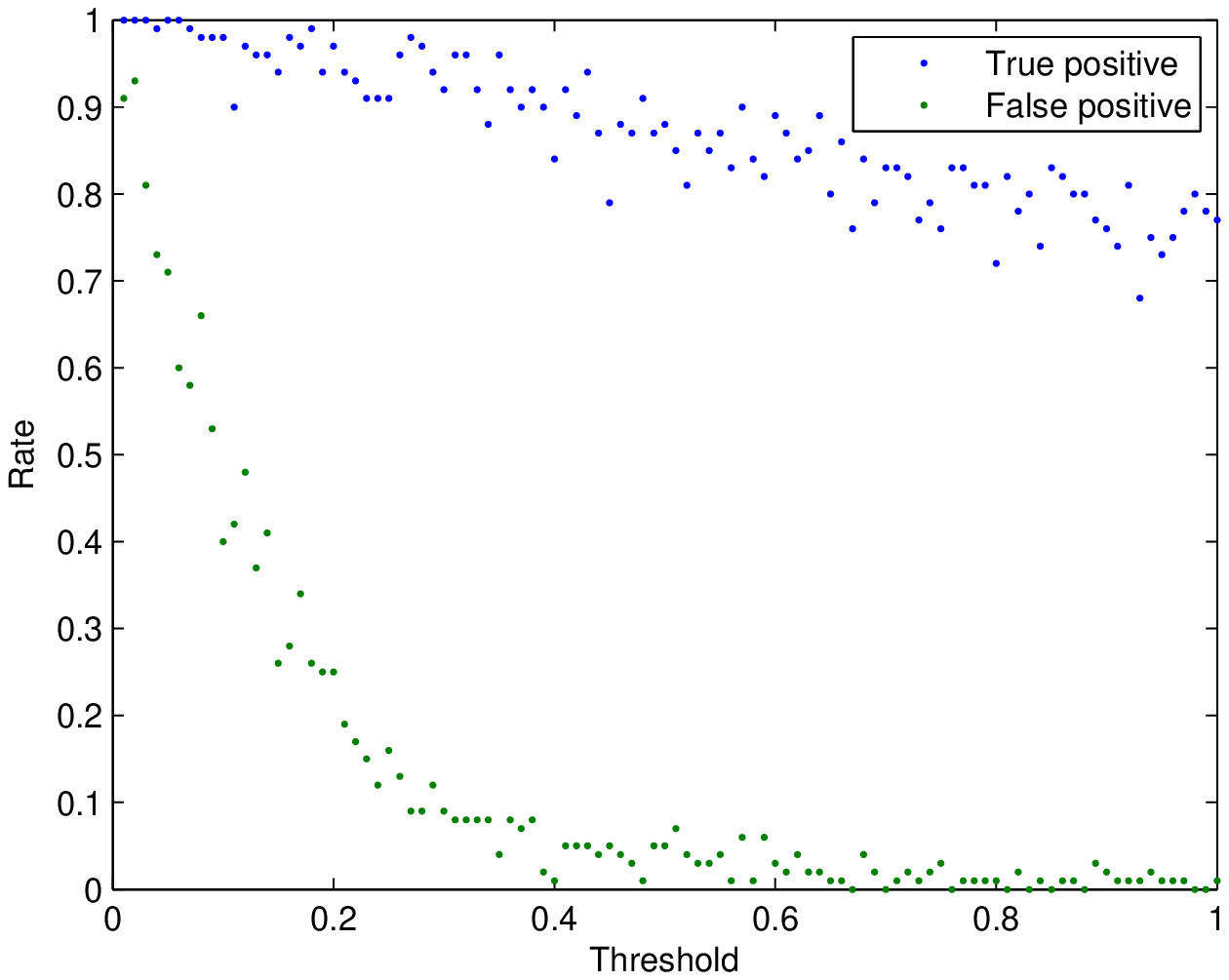} }} 
        \caption{True positive rate and false positive rate versus  thresholds}
      \end{figure}
     It shows that both the true positive rate and false positive rate decrease as the threshold goes larger. The largest gap between them occurs when the threshold is around 0.5. The false positive rate becomes almost 0 while we still have an acceptable true positive rate. That justifies the threshold settings in our previous experiments.   In the next section, we are going to conduct experiments on real world data to show the capability of our method on solving real world confounder detection problems.
    \subsection{Real world data}
    We test the method  on datasets from  UCI machine learning repository. Notice that we include a preprocessing step to normalize all variable to unit variance. We put down a remark of this below.
    \begin{remark}
    We normalize data to deal with the scale variations across features. This could avoid dominant values in the covariance estimation. However, it is not recommended by the paper \citep{janzing2017detecting}, since this normalization jointly changes the covariance matrix and the regression vector. It might violate the independence assumption. 
    \end{remark}
    
     Now we proceed to describe the results. The first dataset is the wine taste dataset. The data contains 11 features and 1 score of the wine as: $x_1$: fixed acidity, $x_2$: volatile acidity, $x_3$: citric acid, $x_4$: residual sugar, $x_5$: chlorides, $x_6$: free sulfur dioxide, $x_7$: total sulfur dioxide, $x_8$: density, $x_9$: pH, $x_{10}$: sulphates, $x_{11}$: alcohol. Y is the score. We sample 500 points from the whole dataset and do 200 deviation tests. We have two settings: first one is including all features, second one is dropping $x_{11}$. The results are plotted in figure 7. 
    \begin{figure}[h]%
        \centering
      \subfloat[ $\widehat{\mathcal{D}}(\tilde{a}_n)$ with all features included ]{{\includegraphics[width=.5\textwidth,height =.4\textwidth]{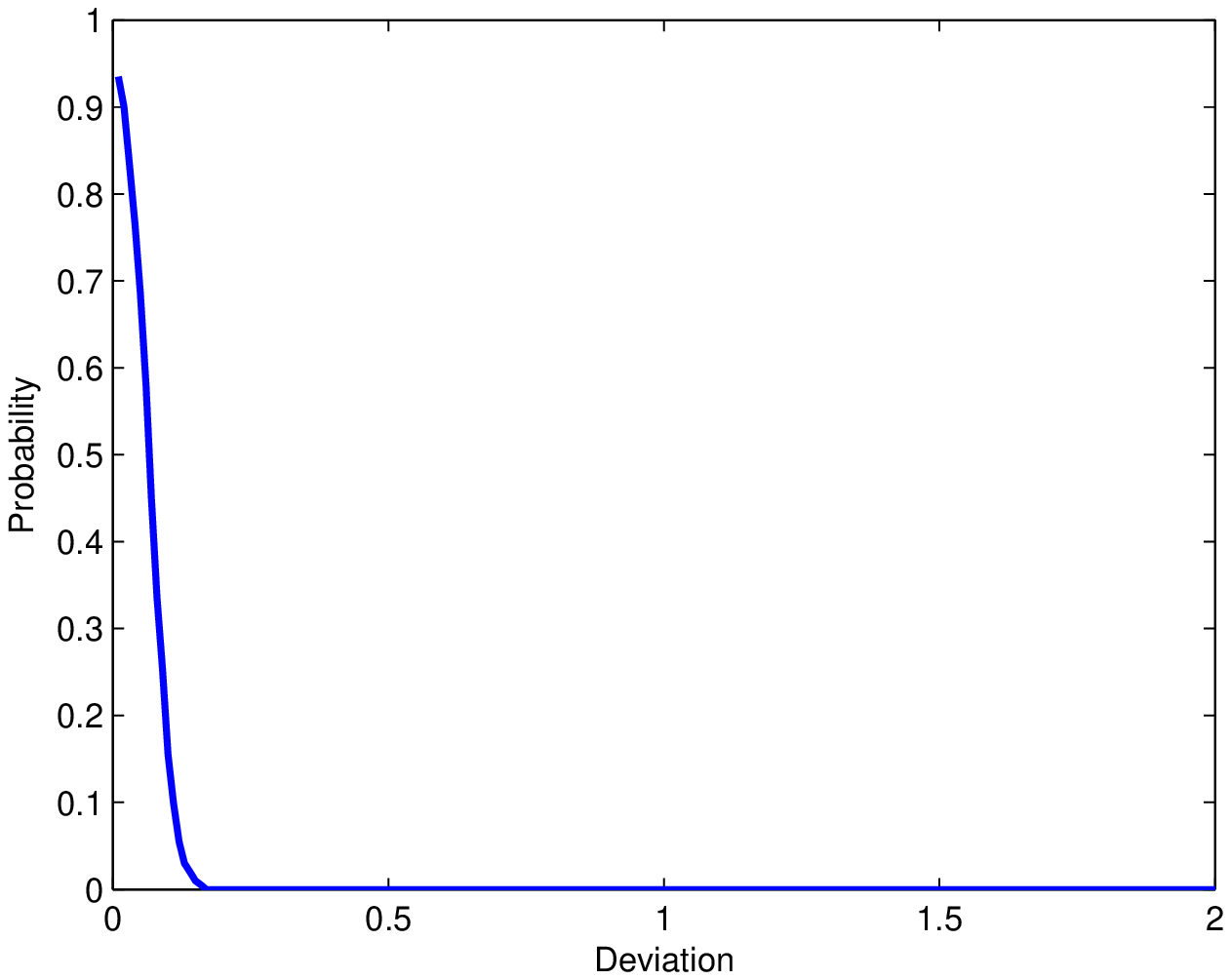} }}
      \subfloat[ $\widehat{\mathcal{D}}(\tilde{a}_n)$ with feature $x_{11}$ dropped ]{{\includegraphics[width=.5\textwidth,height =.4\textwidth]{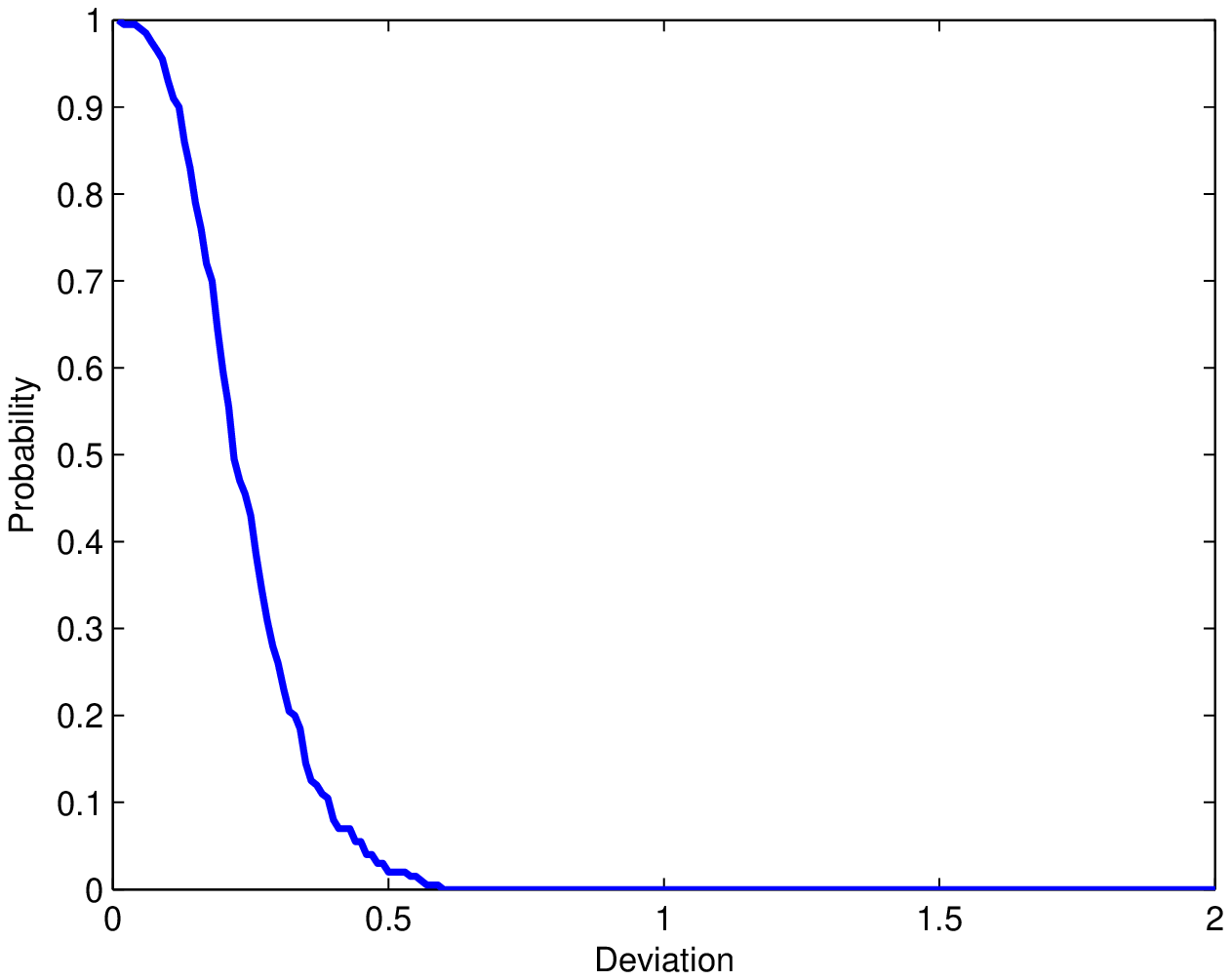} }}  
       \caption{$\widehat{\mathcal{D}}(\tilde{a}_n)$  on taste of wine dataset}
      \end{figure}
    An observation is that dropping $x_{11}$ has caused a clear enlargement of the deviation measure, which indicates that $x_{11}$ is the confounder of the system. This is in consistency with the conclusions of \citep{janzing2017detecting}, which finds the same thing via spectral evidence.
                
     Another dataset is the compressive strength and ingredients of concrete dataset. The target $Y$ is the strength in megapascals. There are 8 features $\{x_1,...,x_8\}$ to predict $Y$. $x_1$: cement, $x_2$: Blast Furnace Slag,  $x_3$: Fly Ash, $x_4$: Water, $x_5$: Superplasticizer, $X_6$: Coarse Aggregate, $x_7$: Fine Aggregate, $x_8$: Age. We sample 500 points from the whole dataset and do 200 deviation tests.  The results are plotted in figure 8.
     \begin{figure}[h]
     \centering
     \includegraphics[width=0.5\textwidth,height =.4\textwidth]{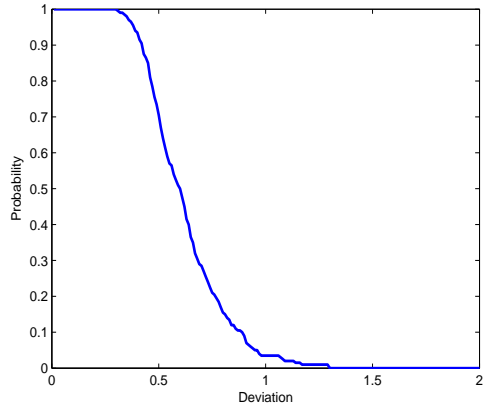}
     \caption{$\widehat{\mathcal{D}}(\tilde{a}_n)$ on concrete compressive strength dataset}
       \end{figure}
      It shows clear evidence that this dataset has hidden confounder between $X_n$ and $Y$. The deviations are not small in general, which may indicate an obvious confounding effect. This is, in some sense, in consistency with the findings by applying the method by \citep{janzing2017detecting}, which reports a significant $\beta^*$ as the evidence for clear confounding. 
       
        We also test out method on the Indian liver patient dataset. The target $Y$ is the indicator of liver or non liver patient. There are 10 features $\{x_1,...,x_{10}\}$ to predict $Y$. $x_1$: age of the patient, $x_2$: Gender,  $x_3$: Total Bilirubin, $x_4$: Direct Bilirubin, $x_5$: Alkaline Phosphotase, $x_6$: Alamine Aminotransferase, $x_7$: Aspartate Aminotransferase, $x_8$: Total Protiens, $x_9$: Albumin, $x_{10}$: Albumin and Globulin Ratio.  We sample 500 points from the whole dataset and do 200 deviation tests. We have two settings: first one is including all features, second one is dropping $x_1$ to $x_{4}$. The results are plotted in figure 9. Dropping the features results in a slightly larger deviation, which may indicate that $x_1$ to $x_4$ weakly confound $Y$ and other features.
  \begin{figure}[h]%
      \centering
  \subfloat[ $\widehat{\mathcal{D}}(\tilde{a}_n)$ with all features included ]{{\includegraphics[width=.5\textwidth,height =.4\textwidth]{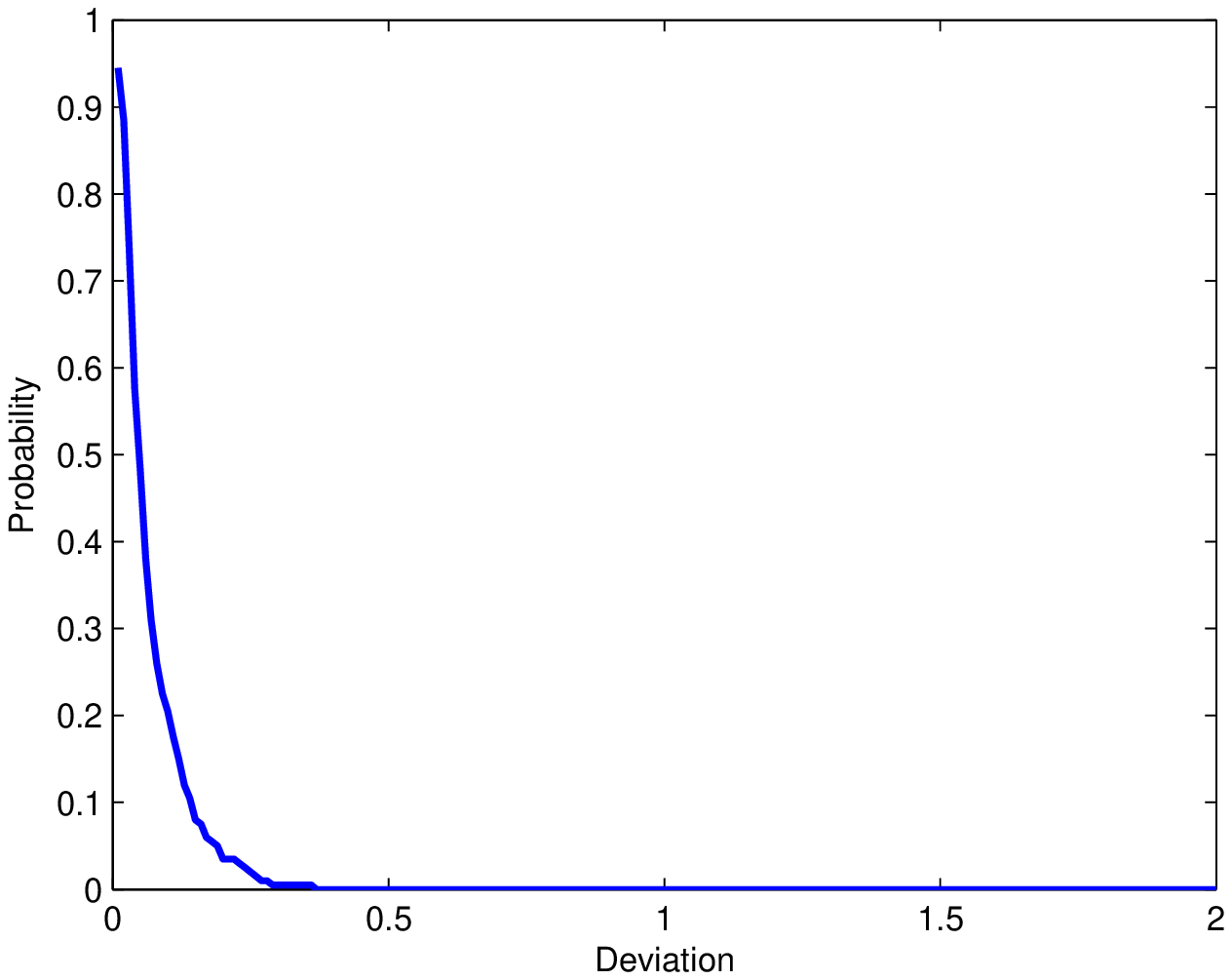} }}
       \subfloat[ $\widehat{\mathcal{D}}(\tilde{a}_n)$ with feature $x_1$ to $x_4$ dropped ]{{\includegraphics[width=.5\textwidth,height =.4\textwidth]{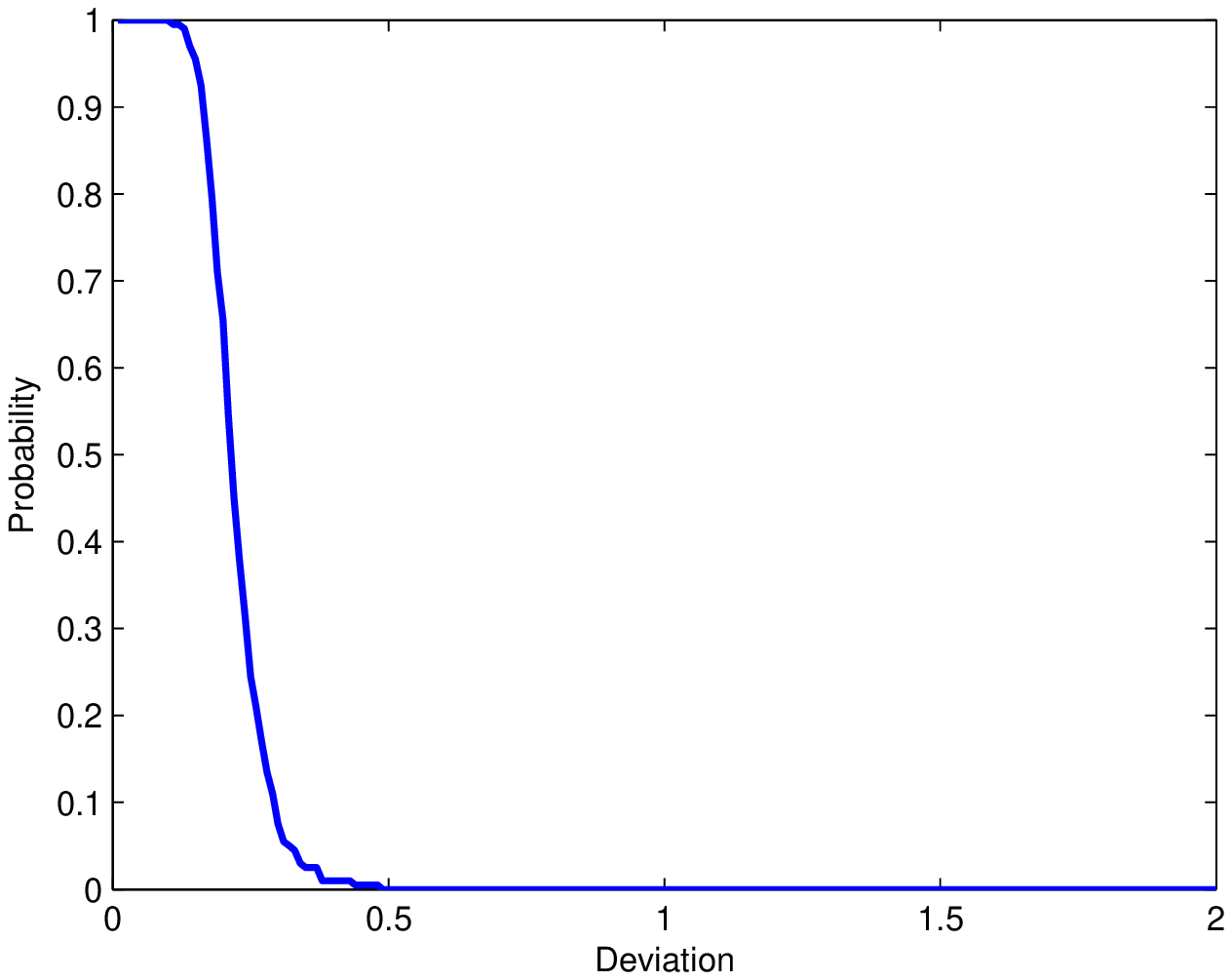} }}  
       \caption{$\widehat{\mathcal{D}}(\tilde{a}_n)$  on Indian liver patient dataset}
\end{figure}
             \section{Conclusions}
                    In this paper, we propose a confounder detection method for high dimensional linear models, when certain model assumptions are satisfied. It relies on the property that the first moment of $\tilde{a}_n$ induced spectral measure coincides with that of a uniform spectral measure (both on $\Sigma_{X_n}$) in purely causal cases, while the two moments often differ in the presence of a confounder. We hope that our  method, modified from the spectral measure pattern matching, could provide people with a simplified yet effective approach for confounding detection. Future work could be  extending the method to work in small sample size cases where estimations of regression vectors and covariance matrix  are inaccurate, and in nonlinear models.     
            \section*{Acknowledgments}
                The authors want to thank the editor and the anonymous reviewers for helpful comments.

   \bibliographystyle{apalike}
\bibliography{nc_paper}

\begin{thebibliography}{}

\bibitem[Bunch et~al., 1978]{bunch1978rank}
Bunch, J.~R., Nielsen, C.~P., and Sorensen, D.~C. (1978).
\newblock Rank-one modification of the symmetric eigenproblem.
\newblock {\em Numerische Mathematik}, 31(1):31--48.

\bibitem[Hoyer et~al., 2008]{hoyer2008estimation}
Hoyer, P.~O., Shimizu, S., Kerminen, A.~J., and Palviainen, M. (2008).
\newblock Estimation of causal effects using linear non-gaussian causal models
  with hidden variables.
\newblock {\em International Journal of Approximate Reasoning}, 49(2):362--378.

\bibitem[Hyv{\"a}rinen and Smith, 2013]{pwlingam}
Hyv{\"a}rinen, A. and Smith, S.~M. (2013).
\newblock Pairwise likelihood ratios for estimation of non-gaussian structural
  equation models.
\newblock {\em Journal of Machine Learning Research}, 14:111--152.

\bibitem[Hyv{\"a}rinen et~al., 2010]{estimationSEM}
Hyv{\"a}rinen, A., Zhang, K., Shimizu, S., and Hoyer, P.~O. (2010).
\newblock Estimation of a structural vector autoregression model using
  non-gaussianity.
\newblock {\em Journal of Machine Learning Research}, 11:1709--1731.

\bibitem[Janzing et~al., 2010]{janzing2010telling}
Janzing, D., Hoyer, P., and Sch{\"o}lkopf, B. (2010).
\newblock Telling cause from effect based on high-dimensional observations.
\newblock In {\em Proceedings of the 27th International Conference on Machine
  Learning (ICML)}, pages 479--486.

\bibitem[Janzing and Schoelkopf, 2017]{janzing2017detecting}
Janzing, D. and Schoelkopf, B. (2017).
\newblock Detecting confounding in multivariate linear models via spectral
  analysis.
\newblock {\em to appear in Journal of Causal Inference, arXiv preprint
  arXiv:1704.01430}.

\bibitem[Janzing and Sch{\"o}lkopf, 2010]{janzing2010causal}
Janzing, D. and Sch{\"o}lkopf, B. (2010).
\newblock Causal inference using the algorithmic markov condition.
\newblock {\em IEEE Transactions on Information Theory}, 56(10):5168--5194.

\bibitem[Lemeire and Janzing, 2013]{lemeire2013replacing}
Lemeire, J. and Janzing, D. (2013).
\newblock Replacing causal faithfulness with algorithmic independence of
  conditionals.
\newblock {\em Minds and Machines}, 23(2):227--249.

\bibitem[Liu and Chan, 2016a]{liu2016causald}
Liu, F. and Chan, L. (2016a).
\newblock Causal discovery on discrete data with extensions to mixture model.
\newblock {\em ACM Transactions on Intelligent Systems and Technology},
  7(2):21:1 -- 19.

\bibitem[Liu and Chan, 2016b]{liu2016causal}
Liu, F. and Chan, L. (2016b).
\newblock Causal inference on discrete data via estimating distance
  correlations.
\newblock {\em Neural Computation}, 28(5):801--814.

\bibitem[Liu and Chan, 2017]{liu2017causal}
Liu, F. and Chan, L. (2017).
\newblock Causal inference on multidimensional data using free probability
  theory.
\newblock {\em IEEE Transactions on Neural Networks and Learning Systems, to
  appear}.

\bibitem[Marton et~al., 1996]{marton1996bounding}
Marton, K. et~al. (1996).
\newblock Bounding d distance by informational divergence: a method to prove
  measure concentration.
\newblock {\em The Annals of Probability}, 24(2):857--866.

\bibitem[Popescu et~al., 2006]{popescu2006entanglement}
Popescu, S., Short, A.~J., and Winter, A. (2006).
\newblock Entanglement and the foundations of statistical mechanics.
\newblock {\em Nature Physics}, 2(11):754--758.

\bibitem[Shiffman and Zelditch, 2003]{shiffman2003random}
Shiffman, B. and Zelditch, S. (2003).
\newblock Random polynomials of high degree and levy concentration of measure.
\newblock {\em arXiv preprint math/0303335}.

\bibitem[Shimizu et~al., 2006]{lingam}
Shimizu, S., Hoyer, P.~O., Hyv{\"a}rinen, A., and Kerminen, A. (2006).
\newblock A linear non-gaussian acyclic model for causal discovery.
\newblock {\em Journal of Machine Learning Research}, 7:2003--2030.

\bibitem[Shimizu et~al., 2011]{directlingam}
Shimizu, S., Inazumi, T., Sogawa, Y., Hyv{\"a}rinen, A., Kawahara, Y., Washio,
  T., Hoyer, P.~O., and Bollen, K. (2011).
\newblock Directlingam: a direct method for learning a linear non-gaussian
  structural equation model.
\newblock {\em Journal of Machine Learning Research}, 12:1225--1248.

\bibitem[Talagrand, 1995]{talagrand1995concentration}
Talagrand, M. (1995).
\newblock Concentration of measure and isoperimetric inequalities in product
  spaces.
\newblock {\em Publications Mathematiques de l'IHES}, 81(1):73--205.

\end{thebibliography}

%
%
%
%
\end{document}